\documentclass[twocolumn]{article}
\usepackage[a4paper, total={180mm, 254mm}]{geometry}
\usepackage{graphicx} % Required for inserting images
\usepackage{authblk}
\usepackage[utf8]{inputenc} % allow utf-8 input
\usepackage{amsmath}
\usepackage{amsthm}
\usepackage{amsfonts}       % blackboard math symbols
\usepackage{amssymb}
\usepackage{xcolor}
\usepackage{enumitem} % to allow letters in enumerations
\usepackage{nicefrac}       % compact symbols for 1/2, etc.
\usepackage{graphicx} % for figures
\usepackage{subcaption}
\usepackage{float} % to fix figures
\usepackage{caption}
\usepackage{array} % for tables with fixed column sizes
\usepackage{longtable} % for multi-page tables
\usepackage[normalem]{ulem}
\usepackage{bm} % bold symbols in math mode
\usepackage{booktabs}
\usepackage{hyperref}
\usepackage[capitalize,nameinlink,noabbrev]{cleveref} % to emulate \autoref style
\usepackage{xspace}
\usepackage[frozencache,cachedir=_minted_cache]{minted} % so it compiles on the arXiv
\usepackage[
backend=biber,
sorting=none,
maxbibnames=99
]{biblatex}
\newcommand{\ra}[1]{\renewcommand{\arraystretch}{#1}}

\newcommand{\wt}{T}

\newcommand{\diag}[1]{\textnormal{diag}\left(#1\right)}

\newtheorem{definition}{Definition}
\newtheorem{proposition}{Proposition}

\crefname{appendix}{App.}{App.}

\definecolor{g}{RGB}{255, 16, 16}

\newcommand{\ie}{i.e., }
\newcommand{\eg}{e.g., }

\DeclareMathOperator{\dop}{do}

\date{} 
\title{\textbf{The Causal Chambers:\\Real Physical Systems as a Testbed for AI Methodology}}

\author[1]{Juan L. Gamella\footnote{Correspondence to Juan L. Gamella\\\nolinkurl{juan.gamella@stat.math.ethz.ch}}}
\author[1]{Jonas Peters}
\author[1]{Peter B\"uhlmann}
\affil[1]{Seminar for Statistics\\ETH Zürich}

\begin{document}

%TC:ignore
% \onecolumn

% \section*{Word Counts}

% This section is \textit{not} included in the word count.

% \detailtexcount{main}

%TC:endignore

\twocolumn
\maketitle

\begin{abstract}
In some fields of AI, machine learning and statistics, the validation of new methods and algorithms is often hindered by the scarcity of suitable real-world datasets. Researchers must often turn to simulated data, which yields limited information about the applicability of the proposed methods to real problems.
As a step forward, we have constructed two devices that allow us to quickly and inexpensively produce large datasets from non-trivial but well-understood physical systems. 
The devices, which we call \emph{causal chambers}, are computer-controlled laboratories that allow us to manipulate and measure an array of variables from these physical systems, providing 
a rich testbed for algorithms from a variety of fields. We illustrate potential applications through a series of case studies in fields such as causal discovery, out-of-distribution generalization, change point detection, independent component analysis, and symbolic regression. For applications to causal inference, the chambers allow us to carefully perform interventions. We also provide and empirically validate a causal model of each chamber, which can be used as ground truth for different tasks. The hardware and software are made open source, and the datasets are publicly available at \href{https://causalchamber.org}{\nolinkurl{causalchamber.org}} or through the Python package \href{https://pypi.org/project/causalchamber/}{\texttt{causalchamber}}.
\end{abstract}

\section{Introduction}

Methodological research in AI, machine learning, and statistics often develops without a concrete application in mind. Many impactful advances in these fields have been made in this way, and there are important theoretical questions that are studied outside the context of a particular application. Crucially, progress also relies on having access to high-quality, real-world datasets, which benefits methodological and theoretical researchers by helping them steer research in meaningful directions, relaxing assumptions that are unlikely to hold in practice, and developing methodologies that may work well on a variety of real-world problems.

However, for some research areas, particularly nascent ones, it can be difficult to find real-world datasets that provide a ground truth suitable to validate new methods and check foundational assumptions that underlie theoretical work. This is because new fields come with new requirements in terms of ground truth, and few or no datasets may have been collected that already satisfy them. For example, for 
most 
sub-fields of causal inference \parencite{spirtes2000causation,pearl2009causality,peters2017elements}, we require data from phenomena whose underlying causal relationships are already exquisitely understood, or for which carefully designed intervention experiments are available. For symbolic regression \parencite{schmidt2009distilling,cava2021contemporary}, the data must follow a known, closed-form mathematical expression, e.g., a natural law in a controlled experimental environment. For the different types
of representation learning \parencite{Locatello2018ChallengingCA,scholkopf2021toward}, we may need data for which there are
some latent ``generating factors'' that we can measure directly.
Such datasets can be difficult to obtain in practice, and few exist for these tasks. As a result, researchers are often limited to synthetic data produced by computer simulations, which may fall short of answering how well a particular method works in practice.

This is where we believe our work can
contribute. We have constructed two physical devices that allow the inexpensive and automated collection of data from two well-understood, real physical systems. The devices, which we call \emph{causal chambers}, consist of a light tunnel and a wind tunnel (\autoref{fig:chambers_diagram}). They are, in essence, computer-controlled laboratories (\autoref{fig:data_collection}) to manipulate and measure different variables of the physical system they contain.

We believe that the chambers are well-suited to substantially improve the validation of methodological advancements across machine learning and statistics, by providing real datasets with a ground truth for fields where such datasets are otherwise scarce or non-existent. This is accomplished through two key properties of the chambers. First, the underlying physical systems are well-understood, in the sense that relationships between most variables are described by first principles and natural laws involving linear, non-linear, and differential equations---see \autoref{apx:physical_effects} and \ref{apx:mechanistic_models} for a detailed description with carefully designed experiments. This allows us to provide ground truths for a variety of tasks, including a causal model of each chamber. Second, we can manipulate the systems in a controlled and automated way, quickly producing vast amounts of data. Furthermore, the chambers produce data of different modalities, including i.i.d., time-series, and image data, allowing us to provide validation tasks for a wide range of methodologies.

To illustrate the practical use of the chambers, we perform case studies in causal discovery, out-of-distribution generalization, change point detection, independent component analysis and symbolic regression---see \autoref{s:applications} and \autoref{fig:benchmarks_1} and \ref{fig:benchmarks_2}. Our choice constitutes only an initial selection, and we believe many other possibilities exist.

Our work complements existing datasets from more complex real-world systems for which a ground truth is not or only partially available \cite[\eg][]{koh2021wilds}, as well as efforts to produce synthetic data that mimics such systems \parencite[\eg][]{gamella2020active,göbler2024textttcausalassembly,cheng2023causaltime,cava2021contemporary,udrescu2020ai,greenfield2010dream4}. While the good performance on the chambers is not guaranteed to carry over to more complex systems, we believe that
the chambers can
serve as a sanity check for foundational assumptions and algorithms that are
intended to work in a variety of settings.

A list of all datasets we currently provide can be found at \href{https://causalchamber.org}{\nolinkurl{causalchamber.org}}, together with a description of the experimental procedures used to collect them. To allow other researchers to build their own chambers, we provide blueprints, component lists, and source code in the repository \href{https://github.com/juangamella/causal-chamber}{\nolinkurl{github.com/juangamella/causal-chamber}}.

\section{The Causal Chambers}
\label{s:chambers}

\begin{figure*}[t]
\centerline{
\includegraphics[width=180mm]{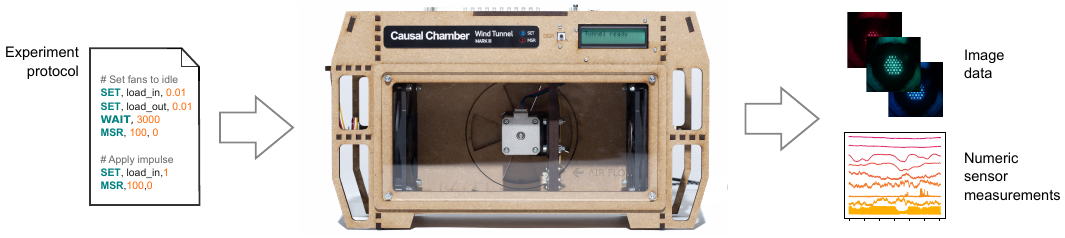}
}
\caption{\textbf{Data collection workflow}. The user provides an \emph{experiment protocol} consisting of step-by-step instructions describing the data collection procedure, which the chamber then carries out without human supervision. The instructions specify when and to which values the actuators and sensor parameters should be set. They also specify when measurements of all variables should be taken and at which frequency, at a maximum of 10 Hz for the light tunnel and 7 Hz for the wind tunnel. Actuators and sensor parameters can also be set automatically by the chamber as a function of other variables in the system, such as sensor measurements. This allows introducing additional complexity for some validation tasks, as described in \autoref{s:testbed}.}
\label{fig:data_collection}
\end{figure*}%
Each chamber is a machine that contains a simple physical system and allows us to measure and manipulate some of its variables. The chambers contain a variety of \emph{sensors}, for example, to measure light intensity or barometric pressure. To manipulate the physical system, \emph{actuators} allow us to control, 
for example, the brightness of a light source or the speed at which fans turn. Each sensor can also be manipulated by modifying some of its parameters, such as the oversampling rate or reference voltage.

Throughout this manuscript, we refer to the actuators and sensor parameters as the \emph{manipulable variables} of the chamber.
A programmable onboard computer controls all sensor parameters and actuators, enabling the chambers to conduct experiments and collect data without human supervision (\autoref{fig:data_collection}). As a result, the chambers can quickly produce vast amounts of data, up to millions of observations or tens of thousands of images per day.

In the remainder of this section, we give an overview of each chamber, its physical system, and some of the measured variables. \autoref{fig:chambers_diagram} provides diagrams of the chambers and their main components, and a detailed description of all variables can be found in \autoref{apx:chamber_variables}.

\subsection{The Wind Tunnel}

The wind tunnel (\autoref{fig:chambers_diagram}ac) is a chamber with two controllable fans that push air through it and barometers that measure air pressure at different locations. A hatch at the back of the chamber controls an additional opening to the outside. A microphone measures the noise level of the fans, and a speaker allows for an independent effect on its reading.

The tunnel provides data from 32 numerical and categorical variables (see \autoref{fig:chambers_data}a for some examples), of which 11 are sensor measurements, and 21 correspond to actuators and sensor parameters that can be manipulated.
For example, we can control the load of the two fans ($L_\text{in}, L_\text{out}$) and measure their speed ($\tilde{\omega}_\text{in}, \tilde{\omega}_\text{out}$), the current they draw ($\tilde{C}_\text{in}, \tilde{C}_\text{out}$), and the resulting air pressure inside the chamber ($\tilde{P}_{\text{dw}}, \tilde{P}_{\text{up}}$)
or at its intake ($\tilde{P}_{\text{int}}$). We can manipulate sensor parameters like the oversampling rate of the barometers ($O_\text{dw}, O_\text{up}, O_\text{int}, O_\text{amb}$) or the timer resolution of the speed sensors ($T_\text{in}, T_\text{out}$), further affecting their measurements. 
In the circuit that drives the speaker, we can manipulate the potentiometers
($A_1, A_2$) that control the amplification, monitoring the resulting signal at different points of the circuit ($\tilde{S}_1, \tilde{S}_2$) and through the microphone output ($\tilde{M}$).

\subsection{The Light Tunnel}

The light tunnel (\autoref{fig:chambers_diagram}bd) is a chamber with a controllable light source at one end and
two linear polarizers mounted on rotating frames. The relative angle between the polarizers dictates how much light passes through them (see \autoref{fig:chambers_data}c and \autoref{fig:chambers_data}e)
and sensors measure the light intensity before, between, and after the polarizers. A camera on the side opposite the light source allows taking images from inside the tunnel.

The tunnel provides image data (\autoref{fig:chambers_data}e) and 41 numerical and categorical variables (\eg \autoref{fig:chambers_data}bcd), of which 32 can be manipulated. For example, we can control the intensity of the light source at three different wavelengths ($R,G$ and $B$) and measure the drawn electric current ($\tilde{C}$). Using motors, we can rotate the polarizer frames to desired angles $(\theta_1, \theta_2)$ and measure the effect on light intensity at different wavelengths ($\tilde{I}_1, \tilde{I}_2, \tilde{I}_3, \tilde{V}_1, \tilde{V}_2, \tilde{V}_3$). 
We can manipulate sensor parameters like the exposure time of the camera ($T_\text{Im}$) or the photodiode used by the light sensors ($D^I_1, D^I_2, D^I_3$), further affecting the readings of these sensors.

\begin{figure*}
\centerline{\includegraphics[width=180mm]{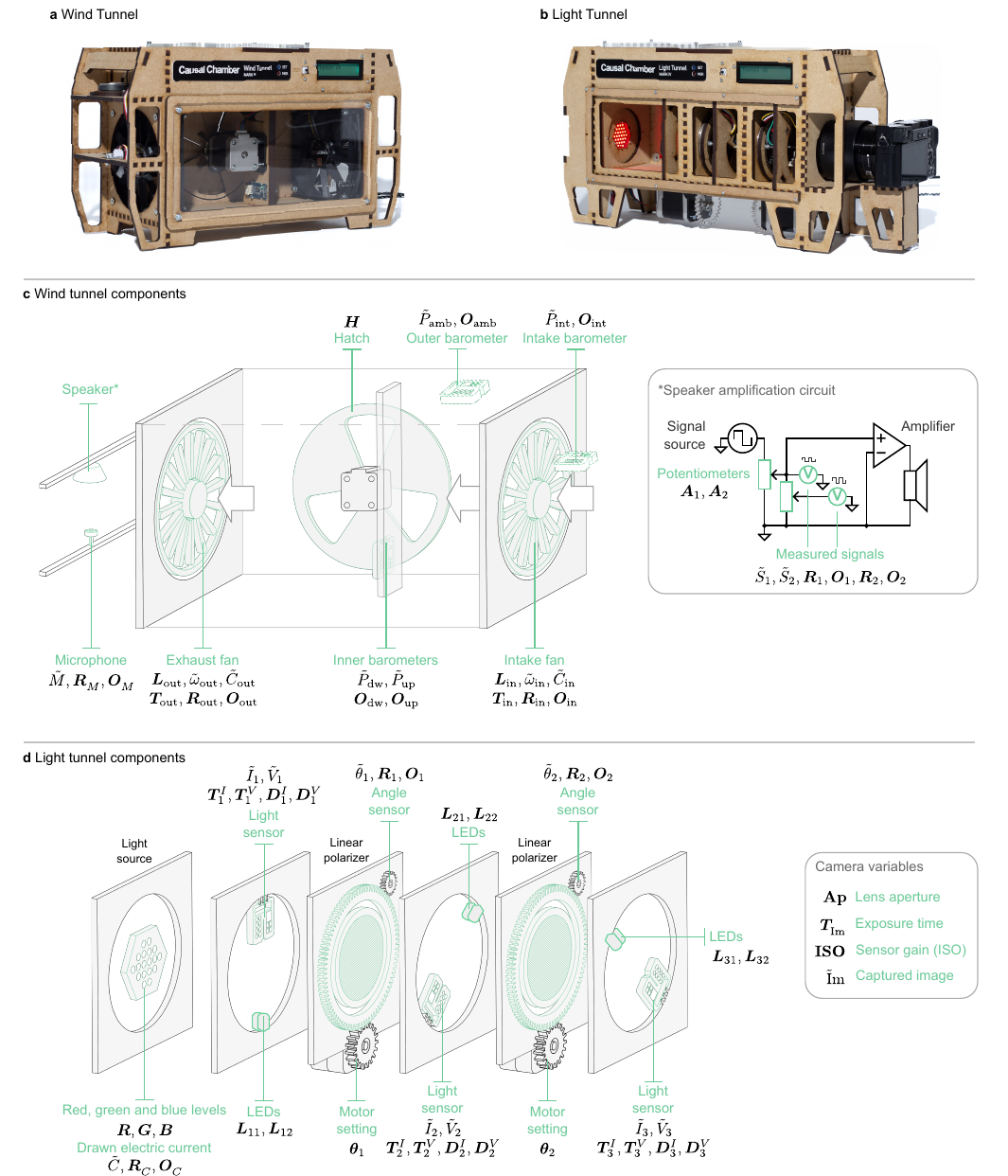}}
\caption{\textbf{The causal chambers.}
(a) The wind tunnel. (b) The light tunnel with the front panel removed to show its inner components. (c,d) Diagrams of the chambers and their main components, including the amplification circuit that drives the speaker of the wind tunnel and the variables for the light tunnel camera. The variables measured by the chambers are displayed in black math print.
Sensor measurements are denoted by a tilde. Manipulable variables, that is, actuators and sensor parameters, are shown in bold symbols (shown as non-bold text elsewhere in the text). A detailed description of each variable is given in \autoref{apx:chamber_variables}.}
\label{fig:chambers_diagram}
\end{figure*}

%%%%%%%%%%%%%%%%%%%%%%%%%%%%%%%%%%%%%%%%%%%%%%%%%%%%%%%%%%%%%%%%%%%%%%
%% SECTION
\section{A Testbed for Algorithms}
\label{s:testbed}

The chambers are designed to provide a testbed for a variety of algorithms from AI, machine learning and statistics. To set up validation tasks, we rely on two key properties of the chambers: that the encapsulated physical system is well understood and that we can manipulate it. For example, by manipulating 
actuators we can evaluate a learned causal model in its prediction of interventional distributions. Or, in the case where the relationships between actuators and sensors are well described by a natural law, we can set up a symbolic regression task where we try to recover it from data. These are some examples of the tasks we set up in our case studies in \autoref{s:applications}, but many other possibilities exist.

\begin{figure*}[t]
\centerline{\includegraphics[width=180mm]{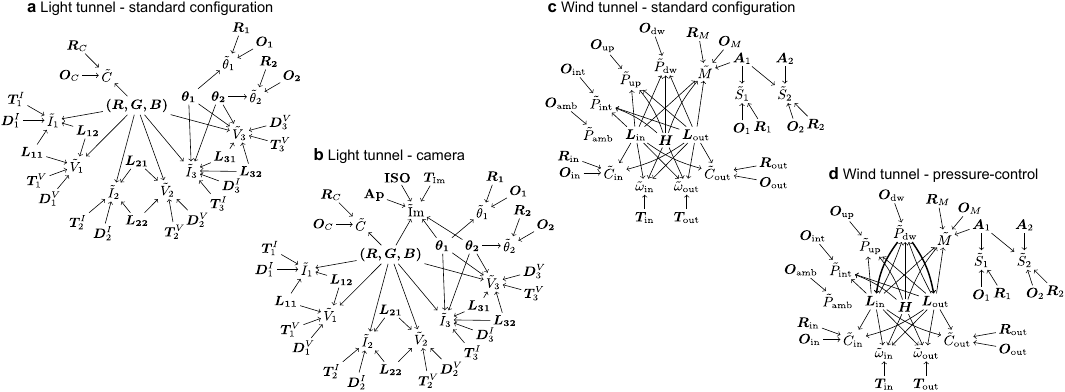}}
\caption{\textbf{Representation of the known effects for different chamber configurations}. Bold symbols correspond to manipulable variables, such as actuators and sensor parameters (shown as non-bold text elsewhere in the text). Sensor measurements are denoted by a tilde. (a,c) Standard configurations of the chambers (b) ``camera'' configuration of the light tunnel, including images from the light tunnel ($\tilde{\text{I}}$m)  and the camera parameters (Ap, ISO, $T_\text{Im}$). (d) ``Pressure-control'' configuration of the wind tunnel, where the load fans $L_\text{in},L_\text{out}$ are set by a control mechanism to maintain the chamber pressure $\tilde{P}_\text{dw}$ at a given level. Each effect (edge in the graph) is described in detail with additional experiments in \autoref{apx:physical_effects}.}
\label{fig:ground_truths}
\end{figure*}%

In \autoref{fig:ground_truths}, we provide a graphical representation of the physical system in each chamber under different configurations, in the form of a directed graph relating its variables. In their most basic form, the chambers operate in the \emph{standard configuration}, where the value of all the actuators and sensor parameters is explicitly given by the user in the experiment protocol (\autoref{fig:data_collection}). The light tunnel operates without the camera to allow the fastest measurement rate.

For additional flexibility in setting up validation tasks, the chambers can also operate in \emph{extended configurations}. For example, these can include additional variables, such as those from the light-tunnel camera (\autoref{fig:ground_truths}b) or additional sensors included in the future. Furthermore, the extended configurations also allow us to assign the value of actuators and sensor parameters as a function of other variables in the system, such as sensor measurements. The assignment is done automatically by the computer onboard the chamber and allows us to introduce additional complexity into the system. For example, the ``pressure-control'' configuration of the wind tunnel (\autoref{fig:ground_truths}d) implements a control mechanism that continuously updates the fan power $(L_\text{in}, L_\text{out})$ to keep the chamber pressure ($\tilde{P}_\text{dw}$) constant. The assignment functions can be any stochastic or deterministic function that can be expressed in the Turing-complete language that controls the chamber computer. This allows us to modify the causal structure underlying the chambers, by introducing additional effects of varying strength between variables. While this yields a vast space of possible configurations, for the moment, we only provide datasets from the four configurations shown in \autoref{fig:ground_truths}.

In \autoref{apx:physical_effects}, we provide a detailed description of all the effects (\ie edges) in \autoref{fig:ground_truths}, based on background knowledge and carefully designed experiments (see \autoref{fig:loads_hatch_effects}-\autoref{fig:angle_sensors}). Furthermore, \autoref{apx:mechanistic_models} contains mechanistic models that describe some of the effects, ranging from simple natural laws to more complex models involving the technical specifications of the actual components. For the more complex processes in the chambers, such as the image capture in the light tunnel or the effects on the wind tunnel pressure, we provide approximate models with increasing degrees of fidelity. In \autoref{fig:benchmarks_2}f, we compare the output of some of these models to measurements gathered from the chambers.

\subsection{Causal Ground Truth}
\label{ss:ground_truth}

For readers with a background in causal inference, the graphs in \autoref{fig:ground_truths} may be reminiscent of causal graphical models \parencite{pearl2009causality,peters2017elements,lauritzen2001causal}. In \autoref{apx:causal_ground_truth}, we formalize a causal interpretation of the graphs and validate them with additional randomized experiments. In short, an edge $X \to Y$ signifies that an intervention on $X$ will change the distribution of subsequent measurements of $Y$. This interpretation allows us to treat the graphs in \autoref{fig:ground_truths} as causal ground truths for a variety of causal inference tasks.

Under our interpretation, the absence of an edge between two variables does not preclude the existence of a causal effect between them. As with most real systems, effects between observed variables
may exist beyond what we know or can validate through the procedures described in this paper, due to a lack of statistical power. Furthermore, there are confounding effects where unmeasured variables simultaneously affect some of the variables in the chambers. For example, variations in the atmospheric pressure outside the chambers simultaneously affect all barometric measurements. We refer the reader to \autoref{apx:causal_ground_truth} for more details.

\section{Case Studies}
\label{s:applications}

\begin{figure*}[t]
\centerline{
\includegraphics[width=180mm]{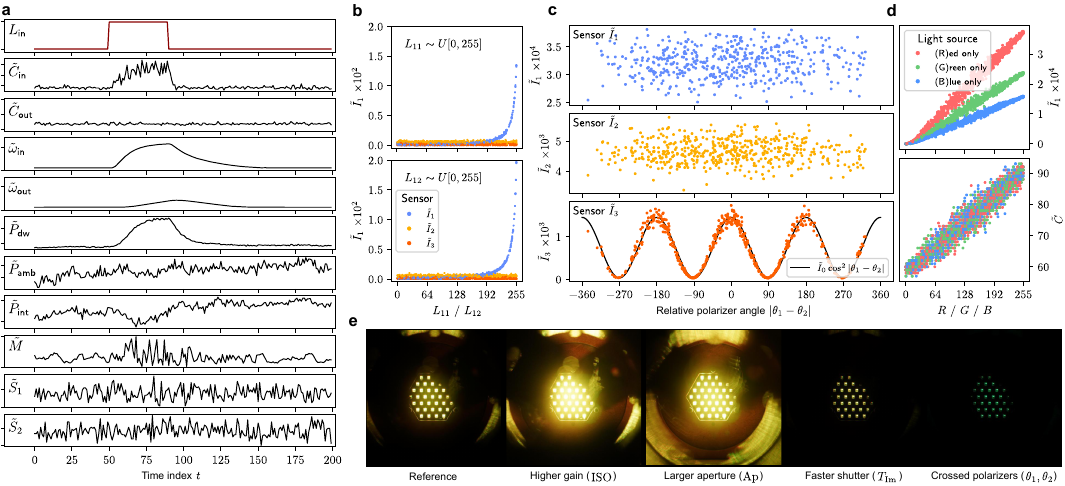}
}
\caption{\textbf{Examples of data produced by the chambers.} (a) Numeric time-series data produced by the wind tunnel under an impulse on the intake fan load ($L_\text{in}$, red), affecting other variables in the system. (b,c) numerical data from the light tunnel illustrating the effect of LED brightness $(L_{11}, L_{12})$ and polarizer angles ($\theta_1, \theta_2$) on the light-intensity readings ($\tilde{I}_1, \tilde{I}_2, \tilde{I}_3$). (d) Effect of the light source setting $(R,G,B)$ on the light intensity reading of the first sensor ($\tilde{I}_1$) and drawn current ($\tilde{C}$). (e) Examples of images from the light tunnel for a fixed light source setting (reference) and interventions on other variables that affect the resulting image.}
\label{fig:chambers_data}
\end{figure*}
We now show, through practical examples, how the chambers can be used to validate algorithms from a variety of fields. As a starting point, we provide a first collection of datasets and set up tasks from a selection of research areas. Our choice is by no means exhaustive, and these case studies are intended as illustrations rather than comprehensive benchmarks. We describe each field and the corresponding tasks below and evaluate the performance of different algorithms, showing the results in \autoref{fig:benchmarks_1} and \autoref{fig:benchmarks_2}.

For each case study, we provide a detailed description of the experimental procedure in \autoref{apx:methods}, together with well-documented code to reproduce the experiments in the paper repository\footnote{\href{https://github.com/juangamella/causal-chamber-paper}{\nolinkurl{github.com/juangamella/causal-chamber-paper}}}. See \autoref{s:data_availability} for details on accessing the datasets.

\paragraph{Causal discovery (\autoref{fig:benchmarks_1}a)}
By offering a causal ground truth and the ability to carry out interventions, the chambers provide an opportunity to validate causal discovery algorithms \parencite[\eg][]{pearl2009causal, peters2017elements, Glymour2019ReviewOC, heinze2018causal, mooij2016distinguishing}, which aim to recover cause-and-effect relationships from data. The chambers provide data suited to validate a wide range of approaches, including those that rely on i.i.d.\ or time-series data \parencite{runge2018causal} with and without instantaneous or lagged causal effects, and causal structures with and without cycles \parencite{Bongers_2021,claassen2023establishing}. We consider the task of recovering the complete causal graph describing the effects in the system \cite{shimizu2006linear, spirtes2000causation, spirtes1999algorithm}, and evaluate algorithms that take different types of data as input: GES \parencite{chickering2002optimal} for purely observational data, UT-IGSP \parencite{squires2020permutation} for interventional data with unknown targets, and PCMCI+ \parencite{runge2020discovering} for time-series data. This constitutes an example selection of methods that is not exhaustive. Performance is measured by the recovery of the ground-truth graph (see \autoref{ss:ground_truth}). The results are shown in \autoref{fig:benchmarks_1}a. In line with their underlying assumptions, both GES and UT-IGSP recover the strong, linear effects from the light-source setting ($R,G,B$) to the light-sensor readings and drawn current. However, both methods struggle with the nonlinear effects of the polarizer angles ($\theta_1, \theta_2$) and the weak effects of the additional LEDs ($L_{11},\ldots,L_{32}$), which are apparent only in the cases when the light-source brightness is low or the polarizers are crossed (\autoref{fig:leds_conditioning}). For the time-series data from the wind tunnel (task a3), PCMCI+ displays a low recall and performs similarly to random guessing, despite the data matching the settings it is intended for.

\paragraph{Out-of-distribution generalization (\autoref{fig:benchmarks_1}b)} 
By manipulating the chamber actuators and sensor parameters, we can induce distribution shifts in a controlled manner. This enables us not only to test the performance of prediction and inference algorithms on data sets with a distribution that differs from the training distribution, but also to investigate under which assumptions on the shifts such methods perform well \parencite{nagarajan2020understanding, geirhos2020shortcut, Rothenhausler2018AnchorRH}. As an illustration, we set up three simple tasks with different data modalities, as
shown in \autoref{fig:benchmarks_1}b.
The first consists of predicting the light-intensity reading $\tilde{I}_1$ from the other numeric variables of the light tunnel. We fit a simple linear regression with an increasing number of predictors and evaluate its predictive performance on data arising from interventions on the light source intensity ($R,G,B$), sensor parameters ($T^I_2, T^V_1, T^V_2, T^V_3$) and polarizer alignment ($\theta_1$). For the second task, we predict the color setting $R,G,B$ of the light source from the images captured by the camera. We employ a small convolutional neural network \cite{FUKUSHIMA1988119}, which we evaluate on shifts induced by changing the distribution of colors, the polarizer angles, and the camera parameters. The goal of the last task is to predict the hatch position $H$ from the pressure curve ($\tilde{P}_\text{dw}$) that results from applying a short impulse to the load $L_\text{in}$ of the intake fan. We fit a simple feed-forward neural network and validate its performance on curves collected under different loads of the exhaust fan $L_\text{out}$, different barometer precision ($O_\text{dw}$), and from a barometer in a different position ($\tilde{P}_\text{up}$). As expected, the performance of the methods degrades under distribution shifts. Even minute changes to the distribution of the inputs, \eg due to an increase in the barometer oversampling rate ($O_\text{dw}\leftarrow8$, see \autoref{fig:osr_comparison}), can make the MLP in task b3 fail. Interestingly, the notion of causal invariance \parencite{peters2016causal} predicts the drop in performance of some models. For example, the mean absolute error (MAE) incurred by predicting the training-set mean (\ie the empty model) remains constant across environments, except in those where the causal parents of the response (see \autoref{fig:ground_truths}a) receive an intervention (\ie $R,G,B$ in tasks b1 and b2). In task b1, the model which includes only causal parents ($\tilde{I}_1 \sim R,G,B$) is most stable across all enviroments, whereas models that include additional (non-causal) variables achieve a better MAE in the training distribution but perform worse in environments where these variables directly or indirectly receive an intervention.

\paragraph{Change point detection (\autoref{fig:benchmarks_1}c)} Change point detection aims to identify abrupt changes or transitions in time-series data or its underlying data-generating process \parencite{truong2020survey}. By manipulating actuators and sensor parameters, we can induce changes in the measurements of the affected sensors, providing real datasets with a known ground truth in terms of change points. To validate offline change point detection algorithms \parencite{truong2020survey,aminikhanghahi2017survey}, we generate time-series data with smooth and abrupt changes of increasing difficulty. We evaluate the non-parametric change-point detection algorithm \texttt{changeforest} \parencite{londschien2023forests}, displaying the results in \autoref{fig:benchmarks_1}c. As expected, the method correctly recovers all change points in the deterministic time-series of the actuator input $L_\text{in}$. For the affected sensors, the method successfully detects abrupt changes in the signal or its regime, but fails to detect more subtle changes, such as those with only a slight effect on the variance (\eg $\tilde{C}_\text{in}$ or $\tilde{M}$ in \autoref{fig:benchmarks_1}c).

\paragraph{Independent component analysis (\autoref{fig:benchmarks_2}d)}
Independent component analysis (ICA) is a family of techniques that treat data as a mixture of latent components and aim to discover a demixing transformation that can accurately recover them \parencite{hyvarinen2023nonlinear,hyvarinen2001independent}. The linear variants of ICA \parencite{hyvarinen2000algorithms,hyvarinen1999fast} are well established, and recent developments in nonlinear ICA have cast it as a framework that holds potential for effectively tackling the challenge of disentanglement in complex data \parencite{hyvarinen2023nonlinear,Locatello2018ChallengingCA}. We propose tasks that consist of recovering (up to indeterminacies such as scaling) the values of independently set actuators from the measurements of the sensors they affect. As a starting point, we set up three tasks, shown in \autoref{fig:benchmarks_2}d: recovering the light-source setting $(R,G,B)$ from the light-intensity measurements ($\tilde{I}_1, \tilde{I}_2, \tilde{I}_3$, $\tilde{V}_1, \tilde{V}_2, \tilde{V}_3$); recovering the fan loads ($L_\text{in}, L_\text{out}$) and hatch position ($H$) from the barometric readings ($\tilde{P}_\text{dw}, \tilde{P}_\text{up},\tilde{P}_\text{amb}, \tilde{P}_\text{int}$), and recovering the configuration of the light source and polarizers ($R,G,B,\theta_1,\theta_2$) from the image data of the light tunnel. The tasks display increasing difficulty in terms of the complexity and dimensionality of the mixing transformation. As a first baseline, we apply FastICA \parencite{hyvarinen1999fast}, which assumes a linear mixing function.
Indeed, the method succeeds in estimating the actuator inputs for task d1 (\autoref{fig:benchmarks_2}), where the mixing function is approximately linear---see \autoref{s:lt_ground_truth}. For the second task (d2), where the effect of the actuators on the sensors is non-linear (see \cref{ss:models_downwind}), the method produces a distorted estimate of the actuators. For the third task (d3), where the mixing function is both non-linear and high-dimensional, the method produces estimates in seemingly little agreement with
the ground-truth signals.

\paragraph{Symbolic regression (\autoref{fig:benchmarks_2}e)} Symbolic regression \parencite{udrescu2020ai,cranmer2020discovering} aims to discover mathematical equations or expressions that best describe the underlying relationships in data, enabling interpretable and compact model representations. A common motivation is the automatic discovery of natural laws from data \parencite{schmidt2009distilling}. Because simple natural laws well describe 
some of the relationships in the chambers, it is possible to provide symbolic regression tasks from real data, and evaluate the performance of such algorithms. As an example, we set up two tasks: recovering Bernoulli's principle, which relates the barometric measurements of the upwind and downwind barometers ($\tilde{P}_\text{up},\tilde{P}_\text{dw}$), and Malus' law, which describes the effect of the linear polarizers ($\theta_1, \theta_2$) on the light-intensity readings of the third sensor ($\tilde{I}_3, \tilde{V}_3$).
More details can be found in \cref{ss:models_bernoulli} and \ref{ss:models_malus}, respectively.
Bernoulli's principle provides a task with a simple ground-truth function but a low signal-to-noise ratio, whereas Malus' law provides a more complex function with weaker noise, representing two common challenges for symbolic regression algorithms. We apply the method described in \parencite{sr2022} and show the results of five runs in \autoref{fig:benchmarks_2}e. The estimated expressions depend strongly on the random initialization of the method, although they all attain a similar $R^2$ score on the data. When we apply the method to synthetic data following Malus' law
with added Gaussian noise, the dependence on the random initialization disappears, and the method returns the correct ground-truth expression in every run (see \autoref{fig:synthetic_symbolic_malus}). This highlights a scenario where synthetic benchmarks may be unreliable for estimating a method's performance in the real world.

\paragraph{Physics-informed machine learning (\autoref{fig:benchmarks_2}f)} Physics-informed machine learning integrates physical laws or domain-specific knowledge into machine learning models to enhance their accuracy and generalizability \parencite{karniadakis2021physics}. To validate such approaches, in \autoref{apx:mechanistic_models}
we provide mechanistic models of several processes in the chambers, derived from first principles. For each process, we consider models of increasing complexity, allowing us to simulate sensor measurements with varying degrees of fidelity. This provides a testbed for simulation-based inference \parencite{sbi} and approaches that exploit potentially misspecified models for inference or generation \parencite{takeishi2021physics,wehenkel2023,yin2021augmenting}. As an illustration, in \autoref{fig:benchmarks_2}f we compare measurements gathered from the chambers with the output of some of these models. In particular, we show the models describing the image capture process of the light tunnel, and the effects of fan loads ($L_\text{in}, L_\text{out}$) and hatch position ($H$) on other wind tunnel variables ($\tilde{P}_\text{dw}, \tilde{\omega}_\text{in}, \tilde{\omega}_\text{out}$). Their description, together with additional models and their outputs, can be found in \autoref{apx:mechanistic_models}.
To facilitate building additional models and simulators, we provide in \autoref{apx:datasheets} the datasheets for every chamber component, detailing its technical specifications and physical properties. 

\begin{figure*}
\centerline{\includegraphics[width=180mm]{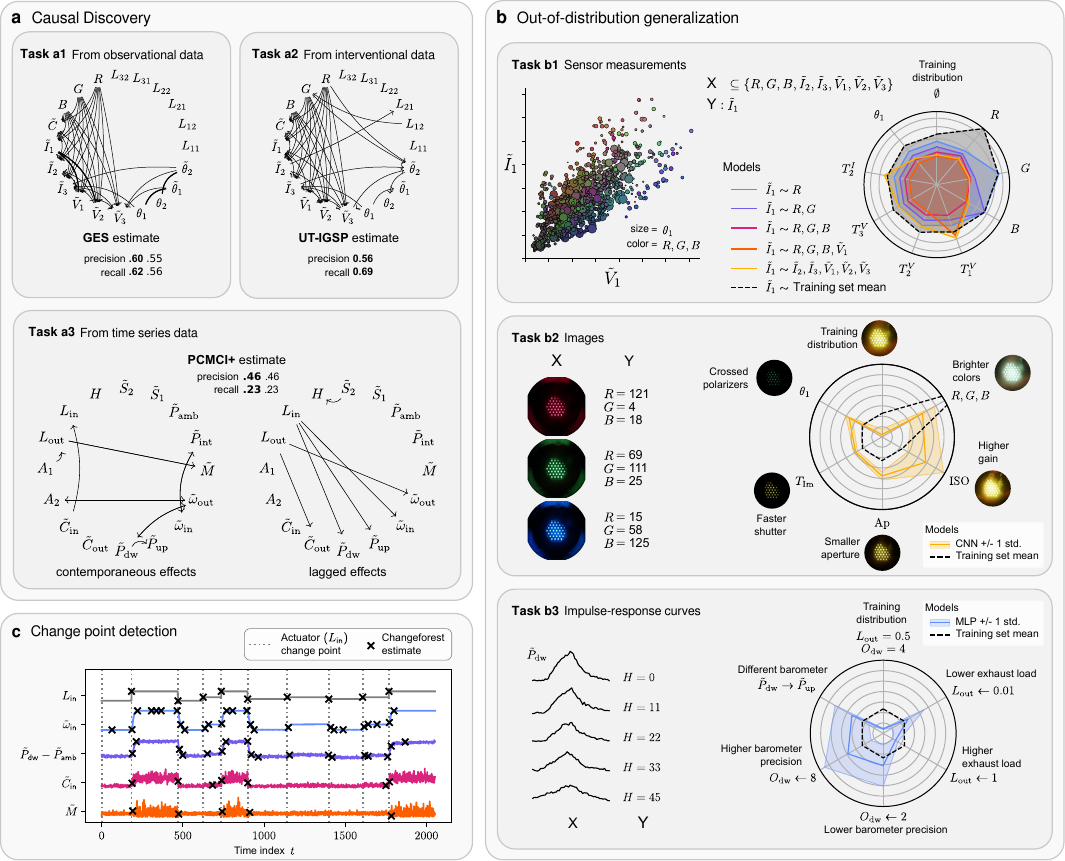}}
\caption{\textbf{Validating algorithms using the chambers (1/2)}. (a) Causal discovery from light-tunnel and wind-tunnel data. The tasks consist of recovering the causal graph from observational data and interventional data from the light tunnel (tasks a1 and a2), and from time-series data from the wind tunnel (task a3). We run a suitable method for each task (GES \parencite{chickering2002optimal}, UT-IGSP with hyperparameter tuning \parencite{squires2020permutation}, and PCMCI+ \parencite{runge2020discovering}, respectively), and evaluate their performance in the recovery of the causal structure of the corresponding ground truth (see \autoref{ss:ground_truth}). GES and PCMCI+ return a set of 12 and 5 plausible graphs, respectively, encoded by a graph with undirected edges \parencite[see \eg][Sec. 2.4.]{chickering2002optimal}. For these methods, we show the precision and recall in the recovery of the directed ground-truth edges (c.f.~equation \ref{eq:prec_recall}, \autoref{apx:methods}) for the best- (bold) and worst-scoring graph in each set. All the graphs returned by PCMCI+ attain the same scores, performing similarly to random guessing. (b) Evaluating the out-of-distribution performance of regression methods. For each task, we try to predict a sensor measurement or actuator value ($Y$) from predictors ($X$) such as numeric measurements (task b1), images (task b2), or impulse-response curves (task b3). We evaluate the predictive performance of each method in terms of its mean absolute error (MAE) on a separate validation set from the training distribution and shifted distributions arising from manipulating the chamber variables. We display the MAE with spider charts, where each axis corresponds to a different setting. As a baseline, we show the MAE incurred when using the average of $Y$ in the training set as prediction (black, dashed). For tasks b2 and b3, the MAE is averaged over 16 random initializations of the model, with error bands corresponding to $\pm1$ standard deviation.
(c) Detecting change points in the time series of different sensor measurements. We change the intake fan load ($L_\text{in}$) at random time points while keeping all other actuators and sensor parameters constant. Because the load affects all the displayed sensors, we take these time points as ground truth (vertical dotted lines) and compare them with the output of the change point detection algorithm (black crosses).
}
\label{fig:benchmarks_1}
\end{figure*}%

\begin{figure*}
\centerline{\includegraphics[width=180mm]{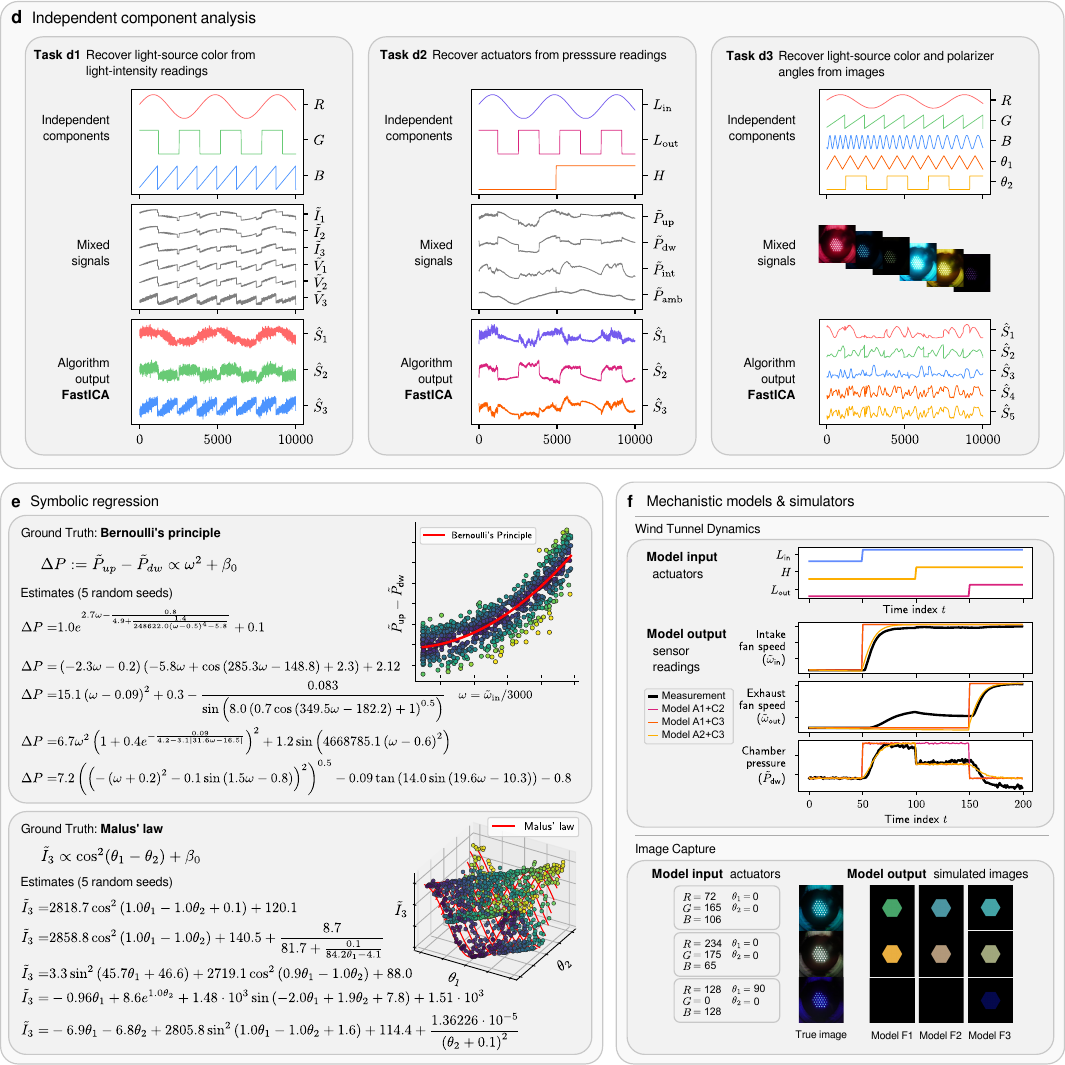}}
\caption{\textbf{Validating algorithms using the chambers (2/2)}. (d) Applying independent component analysis to disentangle the actuator inputs from sensor readings (tasks d1 and d2) and image data (task d3). For each task, we show the actuator values (top), the resulting images and measurements (center), and the sources recovered by the FastICA algorithm \parencite{hyvarinen1999fast} (bottom). For each actuator, we show the recovered source with the highest Pearson correlation coefficient.
(e) Applying symbolic regression to recover (top) Bernoulli's principle from the difference in pressure at the up- and downwind barometers, and (bottom) Malus' law from light-intensity measurements. We show the output of the method described in \cite{sr2022} for five runs with different random initializations. The colors correspond (top) to the residual of the observation w.r.t. the red line and (bottom) to the value of $\theta_2$.
(f) Mechanistic models of the chambers to simulate sensor measurements from actuator inputs with varying degrees of fidelity. For a given set of inputs, we show the outputs of models describing the fan speeds and air pressure in the wind tunnel (top), and the image generation process of the light tunnel (bottom). We compare the model outputs against the images and measurements collected from the chambers (black lines). Higher model numbers imply more complex models with increased fidelity. The models are defined in \autoref{apx:mechanistic_models}.}
\label{fig:benchmarks_2}
\end{figure*}%

\section{Discussion}

We have constructed two devices to collect real-world datasets from well-understood but non-trivial physical systems. The devices provide a testbed beyond simulated data for a variety of empirical inference algorithms in the broad field of AI. To illustrate their use, we have gathered an initial collection of datasets and employed them to perform small case studies in different fields.

The case studies are intended to showcase the flexibility of the chambers in setting up validation tasks; providing exhaustive benchmarks is beyond the scope of this work. However, the mixed performance of algorithms in the case studies suggests that, albeit limited, they can already serve as useful benchmarks for these fields. In some cases, the shortcomings of the methods can be attributed to their underlying assumptions (e.g., tasks a1, a2, d2, and d3). In others, such as tasks e and a3, they highlight a mismatch between performance on synthetic and real data that can lead to an overconfident assessment of a method's capabilities. Task b shows that the chambers provide a principled environment to study phenomena such as causal invariance or the sensitivity of neural networks to small shifts in the distribution of their inputs.

We believe the presented chambers can be used for applications that go beyond the ones we have considered. In particular, the digital control of the chambers makes it possible to validate a variety of active learning, reinforcement learning, and control algorithms.

The chambers are complementary to well-motivated, complex simulators of real phenomena. On the one hand, such simulators allow us to approximate complex systems that are intended application targets, such as mechanisms of the global climate or gene-regulatory networks with hundreds of variables and interactions. On the other hand, it can be difficult (or impossible) to judge if the assumptions used to build these simulators will hold in the real world, and---more importantly---how their violation will affect an algorithm when we use it on real rather than simulated data. Well-understood systems like the chambers provide real-world data without relying on computer simulations and their models and simplifying assumptions. However, the requirement of providing a reliable ground truth necessarily limits the chambers' complexity and size. Therefore, the success of an algorithm on the chambers may not necessarily transfer to larger and more complex systems.

Our aim is that the chambers become a sanity check for algorithms designed to work in a variety of situations. Failures in these testbeds can indicate potential shortcomings in applications to more complex systems. This will allow researchers to test and refine algorithms and methods, and consider fundamental assumptions.

We make all datasets collected from the chambers publicly available, including those used in the case studies of \autoref{s:applications}. Researchers can access them at \href{https://causalchamber.org}{\nolinkurl{causalchamber.org}} and through the Python package \href{https://pypi.org/project/causalchamber/}{\texttt{causalchamber}}, for which we provide an example in \autoref{s:data_availability}. We will continue to expand this dataset repository, and we are open to suggestions of additional experiments that may prove interesting---please reach out to the corresponding author. 

We also provide, in the repository \href{https://github.com/juangamella/causal-chamber}{\nolinkurl{github.com/juangamella/causal-chamber}}, the blueprints and code to allow other researchers to build their own chambers. We believe this to be a key contribution of our work. In first place, the datasets we currently provide amount to only a small fraction of all experiments that are possible with the chambers, and are unlikely to cover many other interesting applications.
Furthermore, having direct access to the chambers is crucial for validation tasks in areas such as active learning, reinforcement learning, and control algorithms. In second place, we hope these resources can be used as a starting point to build chambers around other well-understood systems that prove valuable for the validation of AI methodology.

\section{Data availability}
\label{s:data_availability}
All datasets can be downloaded from the repository at \href{https://causalchamber.org}{\nolinkurl{causalchamber.org}}. The identifier for each dataset used in the case studies is specified in \autoref{apx:methods}. We also provide a Python API to directly download and import the datasets into your code, through the Python package \href{https://pypi.org/project/causalchamber/}{\texttt{causalchamber}}. The package can be installed via pip, \eg by running
\begin{minted}{bash}
  >>> pip install causalchamber
\end{minted}
in an appropriate shell. Datasets can then be accessed directly from Python code. For example, to access the data from Malus' law in the symbolic regression task of \autoref{fig:benchmarks_2}e:
\begin{minted}{python}
from causalchamber.datasets import Dataset

# Download the dataset
#   and store it in the current directory
dataset = Dataset(name='lt_malus_v1', root='./')F

# Load the observations
experiment = dataset.get_experiment('white_255')
df = experiment.as_pandas_dataframe()
\end{minted}
Further examples can be found at \href{https://causalchamber.org}{\nolinkurl{causalchamber.org}}, together with a list of all currently available datasets.

\section{Code availability}
\label{s:code}
The code to reproduce the case studies and figures can be found in the paper repository at
\href{https://github.com/juangamella/causal-chamber-paper}{\nolinkurl{github.com/juangamella/causal-chamber-paper}}.

\section*{Acknowledgements}
\label{s:ack}

We thank all reviewers for their constructive and insightful comments. We would like to thank Manuel Cherep, Christopher Fuchs, Konstantin Göbler, Christina Heinze-Deml, Jörn Jakobsen, Niklas Pfister, Antoine Wehenkel and Tobias Windisch for their valuable discussions and comments on the manuscript. We would also like to thank Niklas Stolz for his help with the design of the polarizer frames of the light tunnel, and Claudio Linares and Helena B\"orjesson for their help with the diagrams and photographs of the chambers.
J.L. Gamella and P. B\"uhlmann have received funding from the European Research Council (ERC) under the European Union’s Horizon 2020 research and innovation program (grant agreement No. 786461).

\onecolumn

\printbibliography

\newpage

%%%%%%%%%%%%%%%%%%%%%%%%%%%%%%%%%%%%%%%%%%%%%%%%%%%%%%%%%%%%%%%%%%%%%%%%%%%%%%%%%
%%%%%%%%%%%%%%%%%%%%%%%%%%%%%%%%%%%%%%%%%%%%%%%%%%%%%%%%%%%%%%%%%%%%%%%%%%%%%%%%%
%%%%%%%%%%%%%%%%%%%%%%%%%%%%%%%%%%%%%%%%%%%%%%%%%%%%%%%%%%%%%%%%%%%%%%%%%%%%%%%%%
%% APPENDIX

\appendix
\renewcommand{\partname}{Appendix}
\renewcommand\thepart{\Roman{part}}
\renewcommand{\thesection}{\thepart.\arabic{section}}
\renewcommand*{\theHsection}{\thepart\the\value{section}}

%%%%%%%%%%%%%%%%%%%%%%%%%%%%%%%%%%%%%%%%%%%%%%%%%%%%%%%%%%%%%%%%%%%%%%%%%%%%%%%%%
%% SECTION
\part{Methods}
\label{apx:methods}
\setcounter{section}{0}

In this appendix, we provide a brief description of the experimental setup for each case study in \autoref{s:applications}, together with a link to the corresponding datasets at \href{https://causalchamber.org}{\nolinkurl{causalchamber.org}}, and to the corresponding code in the paper repository at \href{https://github.com/juangamella/causal-chamber-paper}{\nolinkurl{github.com/juangamella/causal-chamber-paper}}.

%%%%%%%%%%%%%%%%%%%%%%%%%%%%%%%%%%%%%%%%%%%%%%%%%%%%%%%%%%%%%%%%%%%%%%%%%%%%%%%%%%%%%%%
\section*{Case Study: Causal Discovery}

All the methods we evaluate in this case study return a directed acyclic graph (DAG) (or a set of them) as an estimate. Given a single DAG estimate $\hat{G} := (V, \hat{E})$ and a ground-truth graph $G^\star := (V, E^\star)$, we compute the precision $P$ and recall $R$ in terms of directed edge recovery as
\begin{align}
\label{eq:prec_recall}
P := \frac{\hat{E} \cap E^\star}{|\hat{E}|}\text{ and } R := \frac{\hat{E} \cap E^\star}{|E^\star|},    
\end{align}
where $\hat{E}$ and $E^\star$ are the sets of directed edges in $\hat{G}$ and $G^\star$, respectively.
If a method outputs several DAGs, we compute $P$ and $R$ for each element in this set.

\subsection*{Task a1: Observational Data}
\begin{tabular}{@{}rl@{}}
Dataset&\href{https://github.com/juangamella/causal-chamber/tree/master/datasets/lt_interventions_standard_v1}{\nolinkurl{lt_interventions_standard_v1}}\\
Code&\href{https://github.com/juangamella/causal-chamber-paper/blob/main/case_studies/causal_discovery_iid.ipynb}{\nolinkurl{case_studies/causal_discovery_iid.ipynb}}\\\\
\end{tabular}

As input for GES, we take $10000$ observations from a subset of the variables (see \autoref{fig:benchmarks_1}) in the \nolinkurl{uniform_reference} experiment of the \href{https://github.com/juangamella/causal-chamber/tree/master/datasets/lt_interventions_standard_v1}{\nolinkurl{lt_interventions_standard_v1}} dataset. As score for the algorithm, we use the BIC score with a Gaussian likelihood. GES returns the Markov equivalence class of the estimated data-generating graph, and for each graph we compute the corresponding precision and recall in the recovery of the edges in the ground-truth graph.

\subsection*{Task a2: Interventional Data}
\begin{tabular}{@{}rl@{}}
Dataset&\href{https://github.com/juangamella/causal-chamber/tree/master/datasets/lt_interventions_standard_v1}{\nolinkurl{lt_interventions_standard_v1}}\\
Code&\href{https://github.com/juangamella/causal-chamber-paper/blob/main/case_studies/causal_discovery_iid.ipynb}{\nolinkurl{case_studies/causal_discovery_iid.ipynb}}\\\\
\end{tabular}

We consider the same subset of variables as for task a1, taking data from several experiments in the \href{https://github.com/juangamella/causal-chamber/tree/master/datasets/lt_interventions_standard_v1}{\nolinkurl{lt_interventions_standard_v1}} dataset as input for UT-IGSP \parencite{squires2020permutation}. As ``observational data'', we take the 10000 observations from the \nolinkurl{uniform_reference} experiment; as ``interventional data'' we take 1000 observations from each experiment where the considered variables receive an intervention---see the accompanying code for the experiment names. For the conditional independence and invariance tests, we use the default Gaussian tests implemented in the Python package of UT-IGSP, and run the algorithm at different significance levels $(\alpha, \beta) \in [10^{-4}, 10^{-2}]^2$. We show the result for $\alpha=0.008, \beta=0.009$, which performs best in terms of both precision and recall \eqref{eq:prec_recall}.

\subsection*{Task a3: Time-series Data}
\begin{tabular}{@{}rl@{}}
Dataset&\href{https://github.com/juangamella/causal-chamber/tree/master/datasets/wt_walks_v1}{\nolinkurl{wt_walks_v1}}\\
Code&\href{https://github.com/juangamella/causal-chamber-paper/blob/main/case_studies/causal_discovery_time.ipynb}{\nolinkurl{case_studies/causal_discovery_time_series.ipynb}}\\\\
\end{tabular}

As input to PCMCI+ \parencite{runge2020discovering}, we take $10000$ observations from a subset of the variables in the \nolinkurl{actuators_random_walk_1} experiment of the \href{https://github.com/juangamella/causal-chamber/tree/master/datasets/wt_walks_v1}{\nolinkurl{wt_walks_v1}} dataset. We run the method with partial correlation tests at significance level $\alpha=1e-2$ and a maximum of $10$ lags. From the resulting estimate, we drop edges from a variable to itself and edges for which orientation conflicts arise. We compute the precision and recall \eqref{eq:prec_recall} for each of the two graphs in the resulting equivalence class.

\label{ss:methods_causal_discovery}

%%%%%%%%%%%%%%%%%%%%%%%%%%%%%%%%%%%%%%%%%%%%%%%%%%%%%%%%%%%%%%%%%%%%%%%%%%%%%%%%%%%%%%%
\section*{Case Study: Out-of-distribution generalization}
\subsection*{Task b1: Regression from sensor measurements}
\begin{tabular}{@{}rl@{}}
Dataset&\href{https://github.com/juangamella/causal-chamber/tree/main/datasets/lt_interventions_standard_v1}{\nolinkurl{lt_interventions_standard_v1}}\\
Code&\href{https://github.com/juangamella/causal-chamber-paper/blob/main/case_studies/ood_sensors.ipynb}{\nolinkurl{case_studies/ood_sensors.ipynb}}\\\\
\end{tabular}

We use the data from several experiments in the \href{https://github.com/juangamella/causal-chamber/tree/main/datasets/lt_interventions_standard_v1}{\nolinkurl{lt_interventions_standard_v1}} dataset. We begin by splitting the observations from the \nolinkurl{uniform_reference} experiment into a training set (100 observations) and a validation set (1000 observations, shown with $\emptyset$ in the spider plot of \autoref{fig:benchmarks_1}b1). As additional validation sets (1000 observations each), we select experiments where the variables $R,G,B,T^V_1,T^V_2,T^V_3,T^I_2$ and $\theta_1$ receive an intervention---see the accompanying code for the experiment names. These validation sets correspond to the additional axes in the spider plot of \autoref{fig:benchmarks_1}b1. On the training set, we fit linear models with intercept using ordinary least squares, with response $\tilde{I}_1$ and different sets of predictors: $\{R\},\{R,G\},\{R,G,B\},\{R,G,B,\tilde{V}_1\}$ and $\{\tilde{I}_2, \tilde{I}_3, \tilde{V}_1, \tilde{V}_2, \tilde{V}_3\}$. As a baseline, we consider the model that predicts the average of $\tilde{I}_1$ in the training set. For each resulting model, we compute the mean absolute error on each of the validation sets. The additional scatter plot in \autoref{fig:benchmarks_1}b1 corresponds to the pooled data across all validation sets.

\subsection*{Task b2: Regression from images}
\begin{tabular}{@{}rl@{}}
Dataset&\href{https://github.com/juangamella/causal-chamber/tree/main/datasets/lt_color_regression_v1}{\nolinkurl{lt_color_regression_v1}}\\
Code&\href{https://github.com/juangamella/causal-chamber-paper/blob/main/case_studies/ood_images.ipynb}{\nolinkurl{case_studies/ood_images.ipynb}}\\\\
\end{tabular}

We use the images from the \href{https://github.com/juangamella/causal-chamber/tree/main/datasets/lt_color_regression_v1}{\nolinkurl{lt_color_regression_v1}} datasets, at a size of $100 \times 100$ pixels. We split the data from the \nolinkurl{reference} experiment into a training and validation set (9000 and 500 observations, respectively). As additional validation sets, we take those arising from shifts in the distribution of the response $R,G,B$ (\nolinkurl{bright_colors} experiment) and from interventions on the parameters of the camera---see the accompanying code for the experiment names. We subsample each of the additional validation sets to a size of $500$ observations.
As a regression model, we employ a small LeNet-like convolutional neural network \parencite{726791}---see the code for more details. As a loss function, we use the mean-squared error in predicting the light-source settings $R,G,B$, which we minimize using stochastic gradient descent. We fit the model a total of 16 times, each with a different random initialization of the network weights. For each resulting model, we compute the mean absolute error on each validation set,
and plot the results in \autoref{fig:benchmarks_1}b2. As baseline, we consider the model that predicts the average of $R,G,B$ in the training set.

\subsection*{Task b3: Regression from impulse-response curves}
\begin{tabular}{@{}rl@{}}
Dataset&\href{https://github.com/juangamella/causal-chamber/tree/main/datasets/wt_intake_impulse_v1}{\nolinkurl{wt_intake_impulse_v1}}\\
Code&\href{https://github.com/juangamella/causal-chamber-paper/blob/main/case_studies/ood_impulses.ipynb}{\nolinkurl{case_studies/ood_impulses.ipynb}}\\\\
\end{tabular}

We use the data from several experiments in the \href{https://github.com/juangamella/causal-chamber/tree/main/datasets/wt_intake_impulse_v1}{\nolinkurl{wt_intake_impulse_v1}} dataset, corresponding to different settings of the exhaust load $L_\text{out}$ and oversampling rates $O_\text{dw}$ of the downwind barometer---see the accompanying code for the experiment names. We split the data from the \nolinkurl{load_out_0.5_osr_downwind_4} experiment into a training and validation set (4000 and 900 observations, respectively). As a regression model, we employ a multi-layer perceptron with an input layer of size 50 (the impulse length), an output layer of size 1, and two additional hidden layers with 200 neurons and ReLu activations. As a loss function, we use the mean-squared error in predicting the hatch position $H$, and train the model using stochastic gradient descent. We fit the model a total of 16 times, each with a different random initialization of the network weights. For each resulting model, we compute the mean absolute error on validation sets from the training distribution and the additional experiments. Each corresponds to the different axes in the spider plot of \autoref{fig:benchmarks_1}b3. As baseline, we consider the model that predicts the average of $H$ in the training set.

%%%%%%%%%%%%%%%%%%%%%%%%%%%%%%%%%%%%%%%%%%%%%%%%%%%%%%%%%%%%%%%%%%%%%%%%%%%%%%%%%%%%%%%
\section*{Case Study: Change point detection}
\begin{tabular}{@{}rl@{}}
Dataset&\href{https://github.com/juangamella/causal-chamber/tree/main/datasets/wt_changepoints_v1}{\nolinkurl{wt_changepoints_v1}}\\
Code&\href{https://github.com/juangamella/causal-chamber-paper/blob/main/case_studies/changepoints.ipynb}{\nolinkurl{case_studies/changepoints.ipynb}}\\\\
\end{tabular}

We take the data from the \nolinkurl{load_in_seed_9} experiment in the \href{https://github.com/juangamella/causal-chamber/tree/main/datasets/wt_changepoints_v1}{\nolinkurl{wt_changepoints_v1}} dataset, and apply the changeforest algorithm \parencite{londschien2023forests} to each of the time-series $L_\text{in}$, $\tilde{\omega}_\text{in}$, $\tilde{P}_\text{dw} - \tilde{P}_\text{amb}$, $\tilde{C}_\text{in}$, and $\tilde{M}$. For the algorithm, we use the ``random\_forest'' method and default hyperparameters---see the accompanying code for details. As ground truth for the changepoints (vertical gray lines in \autoref{fig:benchmarks_1}c), we take the time points where $L_\text{in}$ is set to a new level. In all datasets collected from the chambers, the column \nolinkurl{intervention} takes a value of 1 for the first measurement after an intervention on any of the chamber variables.

%%%%%%%%%%%%%%%%%%%%%%%%%%%%%%%%%%%%%%%%%%%%%%%%%%%%%%%%%%%%%%%%%%%%%%%%%%%%%%%%%%%%%%%
\section*{Case Study: Independent Component Analysis}
\subsection*{Task d1: Recovering light-source color}
\begin{tabular}{@{}rl@{}}
Dataset&\href{https://github.com/juangamella/causal-chamber/tree/main/datasets/lt_walks_v1}{\nolinkurl{lt_walks_v1}}\\
Code&\href{https://github.com/juangamella/causal-chamber-paper/blob/main/case_studies/ica.ipynb}{\nolinkurl{case_studies/ica.ipynb}}\\\\
\end{tabular}

We use the \nolinkurl{color_mix} experiment from the \href{https://github.com/juangamella/causal-chamber/tree/main/datasets/lt_walks_v1}{\nolinkurl{lt_walks_v1}} dataset. As input to the FastICA algorithm \parencite{hyvarinen1999fast} we take the light-intensity measurements $\tilde{I}_1, \tilde{I}_2, \tilde{I}_3$, $\tilde{V}_1, \tilde{V}_2$ and $\tilde{V}_3$, to which we first apply a whitening transformation. We run the algorithm with 6 components (sources). For each ground-truth source ($R,G,B$), we show the recovered signal with the highest Pearson correlation coefficient (in absolute value).

\subsection*{Task d2: Recovering fan loads and hatch position}
\begin{tabular}{@{}rl@{}}
Dataset&\href{https://github.com/juangamella/causal-chamber/tree/main/datasets/wt_walks_v1}{\nolinkurl{wt_walks_v1}}\\
Code&\href{https://github.com/juangamella/causal-chamber-paper/blob/main/case_studies/ica.ipynb}{\nolinkurl{case_studies/ica.ipynb}}\\\\
\end{tabular}

We use the \nolinkurl{loads_hatch_mix_slow} experiment from the \href{https://github.com/juangamella/causal-chamber/tree/main/datasets/wt_walks_v1}{\nolinkurl{wt_walks_v1}} dataset. As input to the FastICA algorithm \parencite{hyvarinen1999fast} we take the barometric pressure measurements $\tilde{P}_\text{dw}, \tilde{P}_\text{up},\tilde{P}_\text{amb}$ and $\tilde{P}_\text{int}$, to which we first apply a whitening transformation. We run the algorithm with 4 components (sources). For each ground-truth source ($L_\text{in},L_\text{out},H$), we show the recovered signal with the highest Pearson correlation coefficient (in absolute value).

\subsection*{Task d3: Recovering actuators from images}
\begin{tabular}{@{}rl@{}}
Dataset&\href{https://github.com/juangamella/causal-chamber/tree/main/datasets/lt_camera_walks_v1}{\nolinkurl{lt_camera_walks_v1}}\\
Code&\href{https://github.com/juangamella/causal-chamber-paper/blob/main/case_studies/ica.ipynb}{\nolinkurl{case_studies/ica.ipynb}}\\\\
\end{tabular}

As input to the FastICA algorithm \parencite{hyvarinen1999fast}, we use the images from the \nolinkurl{actuator_mix} experiment in the \href{https://github.com/juangamella/causal-chamber/tree/main/datasets/lt_camera_walks_v1}{\nolinkurl{lt_camera_walks_v1}} dataset, at a size of $50 \times 50$ pixels. We first flatten the images, so that each pixel becomes an input variable, applying a whitening transformation as an additional pre-processing step. We run the algorithm with 5 components (sources). For each ground-truth source ($R,G,B,\theta_1,\theta_2$), we show the recovered signal with the highest Pearson correlation coefficient (in absolute value).

%%%%%%%%%%%%%%%%%%%%%%%%%%%%%%%%%%%%%%%%%%%%%%%%%%%%%%%%%%%%%%%%%%%%%%%%%%%%%%%%%%%%%%%
\section*{Case Study: Symbolic Regression}
\begin{tabular}{@{}rl@{}}
Datasets&\href{https://github.com/juangamella/causal-chamber/tree/main/datasets/wt_bernoulli_v1}{\nolinkurl{wt_bernoulli_v1}},\href{https://github.com/juangamella/causal-chamber/tree/main/datasets/lt_malus_v1}{\nolinkurl{lt_malus_v1}}\\
Code&\href{https://github.com/juangamella/causal-chamber-paper/blob/main/case_studies/symbolic_regression.ipynb}{\nolinkurl{case_studies/symbolic_regression.ipynb}}\\\\
\end{tabular}

We use a pre-trained version of the model described in \cite{sr2022}---see the code for more details. For both tasks, we employ the same hyperparameters as in the demonstration provided by the authors\footnote{See the demonstration at \url{https://github.com/facebookresearch/symbolicregression/blob/main/Example.ipynb}} and run the algorithm with 5 different random initializations. As input for the first task, we randomly sample 1000 observations from the \nolinkurl{random_loads_intake} experiment in the \href{https://github.com/juangamella/causal-chamber/tree/main/datasets/wt_bernoulli_v1}{\nolinkurl{wt_bernoulli_v1}} dataset; for the second task we use the 1000 observations from the \nolinkurl{white_255} experiment in the \href{https://github.com/juangamella/causal-chamber/tree/main/datasets/lt_malus_v1}{\nolinkurl{lt_malus_v1}} dataset. We show the estimated expressions in \autoref{fig:benchmarks_2}d, rounding the constants to one decimal place.

\section*{Case Study: Mechanistic Models}
\begin{tabular}{@{}rl@{}}
Datasets&\href{https://github.com/juangamella/causal-chamber/tree/main/datasets/wt_test_v1}{\nolinkurl{wt_test_v1}}, \href{https://github.com/juangamella/causal-chamber/tree/main/datasets/lt_camera_test_v1}{\nolinkurl{lt_camera_test_v1}}\\
Code&\href{https://github.com/juangamella/causal-chamber-paper/blob/main/case_studies/mechanistic_models.ipynb}{\nolinkurl{case_studies/mechanistic_models.ipynb}}\\\\
\end{tabular}

We compare the output of some of the models defined in \autoref{apx:mechanistic_models} to actual measurements collected from the chamber. For the wind-tunnel models, we use the \nolinkurl{steps} experiment from the \href{https://github.com/juangamella/causal-chamber/tree/main/datasets/wt_test_v1}{\nolinkurl{wt_test_v1}}, setting their parameters to the values suggested in \cref{s:models_wt}. For the models of the image-capture process of the light tunnel, we take the images from the \nolinkurl{palette} experiment in the \href{https://github.com/juangamella/causal-chamber/tree/main/datasets/lt_camera_test_v1}{\nolinkurl{lt_camera_test_v1}} dataset and use the parameters suggested in \cref{ss:image_capture}. All models are implemented in the \href{https://github.com/juangamella/causal-chamber#mechanistic-models}{\nolinkurl{causalchamber}} Python package---see the accompanying code for examples.

%%%%%%%%%%%%%%%%%%%%%%%%%%%%%%%%%%%%%%%%%%%%%%%%%%%%%%%%%%%%%%%%%%%%%%%%%%%%%%%%%
%% APPENDIX
\newpage
\part{Chamber Variables}
\label{apx:chamber_variables}
\setcounter{section}{0}

This appendix documents the variables in each chamber, \ie the actuators, sensor parameters, and sensor measurements. A description of the physical effects between them can be found in \autoref{apx:physical_effects}.

\section*{Wind Tunnel}

\renewcommand*{\arraystretch}{1.4}
\begin{longtable}{@{}llclp{8.2cm}@{}}
\toprule
Variables  & Value range & Type & Column name & Description\\
\midrule
\endhead
\bottomrule

\caption{Description of the wind tunnel variables, including their symbol, value range, type, and the corresponding column name in the dataset files. Variables are categorized into different types: actuators (A) or sensor parameters (P), which can be directly manipulated, and sensor measurements (S), which cannot be directly manipulated.
For the value range, $n$:$m$ corresponds to the range $\{n,n+1,\ldots,m\}$ and $n$:$m$:$k$ to $\{n,n+k,n+2k,\ldots,m\}$.}
\endfoot
\bottomrule
\caption{Description of the wind tunnel variables, including their symbol, value range, type, and the corresponding column name in the dataset files. Variables are categorized into different types: actuators (A) or sensor parameters (P), which can be directly manipulated, and sensor measurements (S), which cannot be directly manipulated.
For the value range, $n$:$m$ corresponds to the range $\{n,n+1,\ldots,m\}$ and $n$:$m$:$k$ to $\{n,n+k,n+2k,\ldots,m\}$.}
\label{tab:wind_tunnel_vars}
\endlastfoot
%%%%%%%%%%%
$L_{\text{in}}, L_{\text{out}}$
& \begin{tabular}[t]{@{}l@{}}
$[0,1]$\\
(32-bit float)\\
\end{tabular}
& A & \begin{tabular}[t]{@{}l@{}}
\nolinkurl{load_in}\\[-1.2ex]
\nolinkurl{load_out}\\
\end{tabular} & The load of the fans, corresponding to the duty cycle of the pulse-width-modulation (PWM) signal that controls their speed. For more details, see the datasheet for the \nolinkurl{fan} component in \autoref{apx:datasheets}.\\

%%%%%%%%%%%
$\tilde{C}_{\text{in}}, \tilde{C}_{\text{out}}$
&\begin{tabular}[t]{@{}l@{}}
$[0,1023]$\\
(32-bit float)\\
\end{tabular}
& S & \begin{tabular}[t]{@{}l@{}}
\nolinkurl{current_in}\\[-1.2ex]
\nolinkurl{current_out}\\
\end{tabular} & The uncalibrated measurements of the electric current drawn by the fans. The calibrated measurements (in amperes) are given by
$$\tilde{C}_\text{in} \times \frac{\text{vref}(R_\text{in})}{1023 \times 5} \times 2.5 \;\text{ and }\; \tilde{C}_\text{out} \times \frac{\text{vref}(R_\text{out})}{1023 \times 5} \times 2.5,$$
where $\text{vref}(R_\text{in}), \text{vref}(R_\text{out})$ are the reference voltages of the corresponding sensors (see \autoref{tab:vrefs}).\\
%%%%%%%%%%%
$\tilde{\omega}_{\text{in}}, \tilde{\omega}_{\text{out}}$
& \begin{tabular}[t]{@{}l@{}}
$\geq 0$
(32-bit float)
\end{tabular}
& S & \nolinkurl{rpm_in}, \nolinkurl{rpm_out} & The speed of the fans in revolutions per minute.\\
%%%%%%%%%%%
$T_{\text{in}}, T_{\text{out}}$
& \{0,1\}
& P & \nolinkurl{res_in}, \nolinkurl{res_out} & The resolution of the tachometer timer that measures the elapsed time between successive revolutions of the fan, where $1$ corresponds to microseconds and $0$ to milliseconds. Choosing microseconds yields a higher resolution in the fan-speed measurement.\\
%%%%%%%%%%%
\begin{tabular}[t]{@{}l@{}}
$\tilde{P}_{\text{up}}, \tilde{P}_{\text{dw}}$\\
$\tilde{P}_{\text{amb}}, \tilde{P}_{\text{int}}$
\end{tabular}
& \begin{tabular}[t]{@{}l@{}}
$\geq 0$\\
(32-bit float)\\
\end{tabular}
& S & \begin{tabular}[t]{@{}l@{}}
\nolinkurl{pressure_upwind}\\[-1.2ex]
\nolinkurl{pressure_downwind}\\[-1.2ex]
\nolinkurl{pressure_ambient}\\[-1.2ex]
\nolinkurl{pressure_intake}\\
\end{tabular} & The barometric pressure, in pascals, as measured by the different barometers of the chamber. $\tilde{P}_{\text{int}}$ corresponds to the barometer placed at the tunnel intake. $\tilde{P}_{\text{amb}}$ is the ambient pressure measured by the outer barometer. $\tilde{P}_{\text{up}}$ and $\tilde{P}_{\text{dw}}$ correspond to the barometers inside the tunnel, which are placed facing into ($\tilde{P}_{\text{up}}$) and away ($\tilde{P}_{\text{dw}}$) from the airflow. Even at the same height and under the same conditions, the readings of the barometers have an offset resulting from the manufacturing process.\\
%%%%%%%%%%%
$A_1, A_2$
& 0:255
& A & \nolinkurl{pot_1}, \nolinkurl{pot_2} & The wiper position of the digital potentiometers in the speaker amplification circuit (\autoref{fig:chambers_diagram}c).\\
%%%%%%%%%%%
$\tilde{S}_1, \tilde{S}_2$
& \begin{tabular}[t]{@{}l@{}}
$[0,1023]$\\
(32-bit float)\\
\end{tabular}
& S & \begin{tabular}[t]{@{}l@{}}
\nolinkurl{signal_1}\\[-1.2ex]
\nolinkurl{signal_2}\\
\end{tabular} & The amplitude of the signal after the first potentiometer ($\tilde{S}_1$) and after the second potentiometer ($\tilde{S}_2$) in the speaker amplification circuit (\autoref{fig:chambers_diagram}c). The calibrated amplitudes, in volts, are given by
$$\tilde{S}_1 \times \frac{\text{vref}(R_1)}{1023} \;\text{ and }\; \tilde{S}_2 \times \frac{\text{vref}(R_2)}{1023},$$
where $\text{vref}(R_1), \text{vref}(R_2)$ are the reference voltages of the corresponding sensors (see \autoref{tab:vrefs}).\\
%%%%%%%%%%%
$H$
& 0:45:0.1
& A & \nolinkurl{hatch} & The position, in degrees, of the motor controlling the hatch that covers an additional opening of the wind tunnel. The position can be set in increments of $0.1^{\circ}$ degrees, with the hatch being closed at $0^{\circ}$ and fully open at $45^{\circ}$.\\
%%%%%%%%%%%
$\tilde{M}$
& \begin{tabular}[t]{@{}l@{}}
$[0,1023]$\\
(32-bit float)\\
\end{tabular}
& S & \nolinkurl{mic} & The uncalibrated measurement of the sound level captured by the microphone. The calibrated signal amplitude, in volts, is given by
$$\tilde{M} \times \frac{\text{vref}(R_M)}{1023},$$
where $\text{vref}(R_M)$ is the reference voltage of the corresponding sensor (see \autoref{tab:vrefs}).\\
%%%%%%%%%%%
\begin{tabular}[t]{@{}l@{}}
$R_{\text{in}},R_{\text{out}}$ \\
$R_1,R_2$\\
$R_M$
\end{tabular}
&$\{1.1, 2.56, 5\}$
& P & \begin{tabular}[t]{@{}l@{}}
\nolinkurl{v_in}, \nolinkurl{v_out}\\[-1.2ex]
\nolinkurl{v_1}, \nolinkurl{v_2}\\[-1.2ex]
\nolinkurl{v_mic}
\end{tabular} & The reference voltages, in volts, of the sensors used to measure the current ($\tilde{C}_{\text{in}}, \tilde{C}_{\text{out}}$), amplifier ($\tilde{S}_1, \tilde{S}_2$) and microphone signals ($\tilde{M}$), respectively. The values differ slightly from the actual reference voltages seen by the sensors, which can be found in \autoref{tab:vrefs}.\\
%%%%%%%%%%%
\begin{tabular}[t]{@{}l@{}}
$O_{\text{in}},O_{\text{out}}$ \\
$O_1,O_2,O_M$\\
$O_{\text{up}}, O_{\text{dw}}$\\
$O_{\text{amb}}, O_{\text{int}}$
\end{tabular}
& \{1,2,4,8\}
& P & \begin{tabular}[t]{@{}l@{}}
\nolinkurl{osr_in}, \nolinkurl{osr_out}\\[-1.2ex]
\nolinkurl{osr_1}, \nolinkurl{osr_2}\\[-1.2ex]
\nolinkurl{osr_mic}\\[-1.2ex]
\nolinkurl{osr_upwind}\\[-1.2ex]
\nolinkurl{osr_downwind}\\[-1.2ex]
\nolinkurl{osr_ambient}\\[-1.2ex]
\nolinkurl{osr_intake}
\end{tabular} & The oversampling rates when taking measurements of the current ($\tilde{C}_{\text{in}}, \tilde{C}_{\text{out}}$), amplifier ($\tilde{S}_1, \tilde{S}_2$) and microphone signals ($\tilde{M}$), and of air pressure at the different barometers ($\tilde{P}_{\text{up}}, \tilde{P}_{\text{dw}}, \tilde{P}_{\text{amb}}, \tilde{P}_{\text{int}}$). To avoid affecting the overall measurement time, the chambers always take the maximum number of readings (8) and discard excess readings accordingly.\\
\end{longtable}

\section*{Light Tunnel}

\renewcommand*{\arraystretch}{1.4}
\begin{longtable}{@{}llclp{9.5cm}@{}}
\toprule
Variables  & Value range & Type & Column name & Description\\
\midrule
\endhead
\bottomrule
\caption{Description of the light tunnel variables, including their symbol, value range, type, and the corresponding column name in the dataset files. Variables are categorized into different types: actuators (A) or sensor parameters (P), which can be directly manipulated, and sensor measurements (S), which cannot be directly manipulated. For the value range, $n$:$m$ corresponds to the range $\{n,n+1,\ldots,m\}$.}
\endfoot
\bottomrule
\caption{Description of the light tunnel variables, including their symbol, value range, type, and the corresponding column name in the dataset files. Variables are categorized into different types: actuators (A) or sensor parameters (P), which can be directly manipulated, and sensor measurements (S), which cannot be directly manipulated. For the value range, $n$:$m$ corresponds to the range $\{n,n+1,\ldots,m\}$.}
\label{tab:light_tunnel_vars}
\endlastfoot

%%%%%%%%%%%
$R, G, B$
& 0:255
& A & \begin{tabular}[t]{@{}l@{}}
\nolinkurl{red}\\[-1.2ex]
\nolinkurl{green}\\[-1.2ex]
\nolinkurl{blue}
\end{tabular} & The brightness setting of the red, green, and blue LEDs on the main light source. Higher values correspond to higher brightness.\\
%%%%%%%%%%%
$\tilde{C}$
& \begin{tabular}[t]{@{}l@{}}
$[0,1023]$\\
(32-bit float)\\
\end{tabular}
& S & \nolinkurl{current} & The uncalibrated measurement of the electric current drawn by the light source. The calibrated measurement (in amperes) is given by
$$\tilde{C} \times \frac{\text{vref}(R_C)}{1023 \times 5} \times 2.5,$$
where $\text{vref}(R_C)$ is the reference voltage of the corresponding sensor (see \autoref{tab:vrefs}).\\
%%%%%%%%%%%
$\tilde{I}_1, \tilde{I}_2, \tilde{I}_3$
& 0:$2^{16}-1$
& S & \begin{tabular}[t]{@{}l@{}}
\nolinkurl{ir_1}\\[-1.2ex]
\nolinkurl{ir_2}\\[-1.2ex]
\nolinkurl{ir_3}
\end{tabular} & The uncalibrated infrared measurement of the light-intensity sensors, placed before ($\tilde{I}_1$), between ($\tilde{I}_2$), and after ($\tilde{I}_3$) the polarizers, in reference to the light source.\\
%%%%%%%%%%%
\pagebreak
$\tilde{V}_1, \tilde{V}_2, \tilde{V}_3$
& 0:$2^{16}-1$
& S & \begin{tabular}[t]{@{}l@{}}
\nolinkurl{vis_1}\\[-1.2ex]
\nolinkurl{vis_2}\\[-1.2ex]
\nolinkurl{vis_3}
\end{tabular} & The uncalibrated visible-light measurement of the light-intensity sensors, placed before ($\tilde{V}_1$), between ($\tilde{V}_2$), and after ($\tilde{V}_3$) the polarizers, in reference to the light source.\\
%%%%%%%%%%%
$D^I_1, D^I_2, D^I_3$
& $\{0,1,2\}$
& P & \begin{tabular}[t]{@{}l@{}}
\nolinkurl{diode_ir_1}\\[-1.2ex]
\nolinkurl{diode_ir_2}\\[-1.2ex]
\nolinkurl{diode_ir_3}
\end{tabular} & The photodiodes used by the light sensors to take infrared measurements, corresponding to the small $(D^I_j=0)$, medium $(D^I_j=1)$ and large $(D^I_j=2)$ infrared photodiodes onboard.\\
%%%%%%%%%%%
$D^V_1, D^V_2, D^V_3$
& $\{0,1\}$
& P & \begin{tabular}[t]{@{}l@{}}
\nolinkurl{diode_vis_1}\\[-1.2ex]
\nolinkurl{diode_vis_2}\\[-1.2ex]
\nolinkurl{diode_vis_3}
\end{tabular} & The photodiodes used by the light sensors to take visible-light measurements, corresponding to the small $(D^V_j=0)$ and large $(D^V_j=1)$ photodiodes onboard.\\
%%%%%%%%%%%
\begin{tabular}[t]{@{}l@{}}
$T^I_1, T^I_2, T^I_3$\\
$T^V_1, T^V_2, T^V_3$\\
\end{tabular}
& $\{0,1,2,3\}$
& P & \begin{tabular}[t]{@{}l@{}}
\nolinkurl{t_ir_1/2/3}\\[-1.2ex]
\nolinkurl{t_vis_1/2/3}
\end{tabular} & The exposure time of the photodiode during a light-intensity measurement.\\
\begin{tabular}[t]{@{}l@{}}
$L_{11}, L_{12},$ \\
$L_{21}, L_{22},$ \\
$L_{31}, L_{32}$
\end{tabular}
& 0:255
%& \{0,1,...,255\}
& A & \begin{tabular}[t]{@{}l@{}}
\nolinkurl{l_11}, \nolinkurl{l_12}\\[-1.2ex]
\nolinkurl{l_21}, \nolinkurl{l_22}\\[-1.2ex]
\nolinkurl{l_31}, \nolinkurl{l_32}
\end{tabular} & The brightness of LEDs placed by each light-intensity sensor, where $L_{i1}, L_{i2}$ correspond to the two LEDs placed by the $i^\text{th}$ sensor ($\tilde{I}_i, \tilde{V}_i$). The value corresponds to the wiper position of a digital rheostat, which controls the current flowing to the LED. Higher values correspond to increased brightness.\\
%%%%%%%%%%%
$\theta_1, \theta_2$
& -180:180:0.1
& A & \begin{tabular}[t]{@{}l@{}}
\nolinkurl{pol_1}\\[-1.2ex]
\nolinkurl{pol_2}
\end{tabular} & The setting of the stepper motors that control the angle of the polarizer frames, which can be set in increments of 0.1 degrees. Because the mechanism functions without feedback, the actual angle of the polarizers may slightly deviate from this setting due to the imperfect coupling of the mechanical pieces.\\
%%%%%%%%%%%
$\tilde{\theta}_1, \tilde{\theta}_2$
& \begin{tabular}[t]{@{}l@{}}
$[0,1023]$\\
(32-bit float)\\
\end{tabular}
& S & \begin{tabular}[t]{@{}l@{}}
\nolinkurl{angle_1}\\[-1.2ex]
\nolinkurl{angle_2}
\end{tabular} & The position of the polarizers is encoded into a voltage using a rotary potentiometer, which is then read by the control computer to produce the measurements $\tilde{\theta_1}$ and $\tilde{\theta_2}$. Given these measurements, the calibrated angle measurement is given as
$$(\tilde{\theta}_j - Z_j) \times \frac{720}{1023} \times \frac{\text{vref}(R_j)}{5} \;\text{ degrees},$$
where $\text{vref}(R_j)$ is the reference voltage of the corresponding sensor (see \autoref{tab:vrefs}), and $Z_1 = 507, Z_2 = 512$ are the readings at angles $\theta_1 = \theta_2 = 0$ and reference voltages $R_1 = R_2 = 5$.\\
%%%%%%%%%%%
$R_C,R_1,R_2$
&$\{1.1, 2.56, 5\}$
& P & \begin{tabular}[t]{@{}l@{}}
\nolinkurl{v_c}\\[-1.2ex]
\nolinkurl{v_angle_1}\\[-1.2ex]
\nolinkurl{v_angle_2}
\end{tabular} & The reference voltage, in volts, of the sensors used to measure the current ($\tilde{C}$) and polarizer angles ($\tilde{\theta}_1, \tilde{\theta}_2$), respectively. The values differ slightly from the actual reference voltages used by the sensors, which can be found in \autoref{tab:vrefs}.\\
%%%%%%%%%%%
$O_C,O_1,O_2$
& \{1,2,4,8\}
& P & \begin{tabular}[t]{@{}l@{}}
\nolinkurl{osr_c}\\[-1.2ex]
\nolinkurl{osr_angle_1}\\[-1.2ex]
\nolinkurl{osr_angle_2}
\end{tabular} & The oversampling rate of the sensors used to measure the current ($\tilde{C}$) and polarizer angles ($\tilde{\theta}_1, \tilde{\theta}_2$), respectively. To avoid affecting the overall measurement time, the chambers always take the maximum number of readings (8) and discard excess readings accordingly.\\
%%%%%%%%%%%
$\tilde{\text{I}}\text{m}$
& \begin{tabular}[t]{@{}l@{}}
24-bit RGB \\
image
\end{tabular}
& S & \nolinkurl{im} & The color image produced by the camera, with a size of $2000 \times 2000$ pixels and 8 bits per color channel.\\
%%%%%%%%%%%
Ap
&  cf.\ \autoref{tab:camera_values}
& P & \nolinkurl{aperture} & The $f$-number describing the aperture of the camera lens, with higher values corresponding to smaller openings.\\
%%%%%%%%%%%
\pagebreak
ISO
& cf.\ \autoref{tab:camera_values}
& P & \nolinkurl{iso} & The gain of the camera sensor, where higher values correspond to higher sensitivity.\\
%%%%%%%%%%%
$T_\text{Im}$
& cf.\ \autoref{tab:camera_values}
& P & \nolinkurl{shutter_speed} & The shutter speed of the camera, \ie how many seconds the camera sensor is exposed when taking an image.\\
%\end{tabular}
\end{longtable}
%\end{table}

\begin{table}[H]\centering
\ra{1.3}
\begin{tabular}{@{}rcc@{}}\toprule
& Wind Tunnel & Light Tunnel\\
$R$ & $\text{vref}(R)$ & $\text{vref}(R)$\\
\midrule
1.1 & 1.16 & 1.09\\
2.56 & 2.65 & 2.55\\
5 & 5 & 5\\
\bottomrule
\end{tabular}
\caption{Calibrated reference voltages for the wind tunnel ($R = R_{\text{in}},R_{\text{out}}, R_1,R_2, R_M$) and the light tunnel ($R = R_C,R_1,R_2,$). They are estimated using the \nolinkurl{analog_calibration} experiments in the \href{https://github.com/juangamella/causal-chamber/tree/main/datasets/wt_test_v1}{\nolinkurl{wt_test_v1}} and \href{https://github.com/juangamella/causal-chamber/tree/main/datasets/lt_test_v1}{\nolinkurl{lt_test_v1}} datasets.}
\label{tab:vrefs}
\end{table}

\begin{table}[H]\centering
\ra{1.3}
\begin{tabular}{@{}rp{10cm}@{}}\toprule
Variable & Values\\
\midrule
Ap & 1.8, 2.0, 2.2, 2.5, 2.8, 3.2, 3.5, 4.0, 4.5, 5.0, 5.6, 6.3, 6.4, 7.1, 8.0, 9.0, 10, 11, 13, 14, 16, 18, 20, 22\\
ISO & 100, 125, 160, 200, 250, 320, 400, 500, 640, 800, 1000, 1250, 1600, 2000, 2500, 3200, 4000, 5000, 6400, 8000, 10000, 12800, 16000, 20000, 25600, 32000, 40000, 51200\\
$T_\text{Im}$ & 1/200, 1/250, 1/320, 1/400, 1/500, 1/640, 1/800, 1/1000, 1/1250, 1/1600, 1/2000, 1/2500, 1/3200, 1/4000\\
\bottomrule
\end{tabular}
\caption{Possible values for the camera parameters controlling the aperture of the lens (Ap), the sensor gain (ISO), and the shutter speed ($T_\text{Im}$).}
\label{tab:camera_values}
\end{table}
 
%%%%%%%%%%%%%%%%%%%%%%%%%%%%%%%%%%%%%%%%%%%%%%%%%%%%%%%%%%%%%%%%%%%%%%%%%%%%%%%%%
%% APPENDIX
\newpage
\part{Description of Physical Effects}
\label{apx:physical_effects}
\setcounter{section}{0}

We describe the effects that the chamber actuators and sensor parameters have on the measurements of each sensor. The effects correspond to the edges in the ground-truth graphs of the standard configurations in \autoref{fig:ground_truths}, where an edge $A \to B$ denotes that the actuator or sensor parameter $A$ has an effect on the sensor measurement $B$. We justify each effect in terms of the chamber design and the underlying physical principles, further characterizing each effect through additional experiments (\autoref{fig:loads_hatch_effects} to \autoref{fig:polarizer_model}). As discussed in \autoref{ss:ground_truth}, these effects can be understood as causal effects---see \autoref{apx:causal_ground_truth} for an in-depth discussion and additional validation through randomized experiments. Throughout this appendix, we use a short-hand notation for denoting multiple effects \eg $\{A_1,A_2\} \to \{B_1,B_2\}$ refers to the four edges $A_1 \to B_1, A_1 \to B_2, A_2 \to B_1$ and $A_2 \to B_2$.

The code to generate the plots in this section can be found in the Jupyter notebook \href{https://github.com/juangamella/causal-chamber-paper/blob/main/plots_appendices.ipynb}{\nolinkurl{case_studies/plots_appendices.ipynb}} in the paper repository \href{https://github.com/juangamella/causal-chamber-paper}{\nolinkurl{github.com/juangamella/causal-chamber-paper}}.

\section{Wind Tunnel}
\label{s:wt_ground_truth}

\begin{figure}[ht]
\centerline{
\includegraphics[width=159mm]{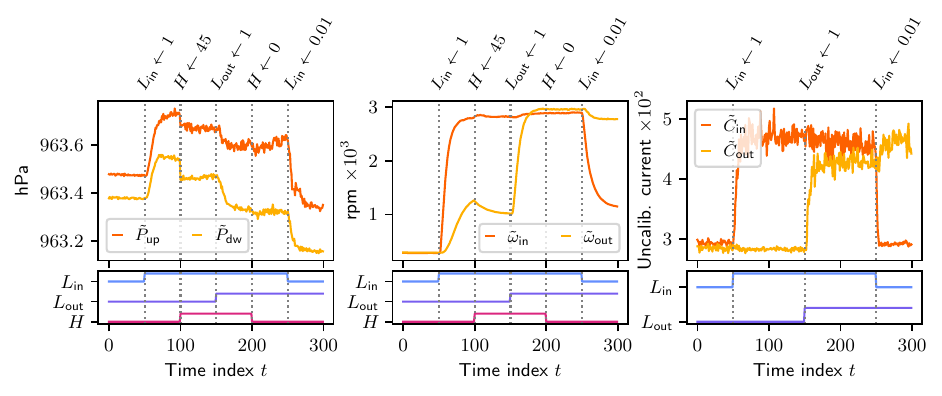}
}
\caption{Time-series data from different sensor measurements under varying inputs to the fan loads ($L_\textnormal{in}, L_\textnormal{out}$) and hatch position ($H$), shown on the bottom row. (left) The measurements $\tilde{P}_\text{up}, \tilde{P}_\text{dw}$ of the inner barometers are affected by both loads and the hatch position. (center) The fan speeds are affected by their corresponding load. Because the fans are placed in tandem, their speed is also affected by the load of the other fan, with the strength of the effect depending on the hatch position (\eg $t=100$). (right) Because they share the same power supply, the load of one fan has a slight effect on the current drawn by the other (\eg $t=250$). The data and corresponding experiment protocol can be found under the \nolinkurl{steps} experiment in the \href{https://github.com/juangamella/causal-chamber/tree/main/datasets/wt_test_v1}{\nolinkurl{wt_test_v1}} dataset.}
\label{fig:loads_hatch_effects}
\end{figure}%

\begin{figure}[ht]
\centerline{
\includegraphics[width=154mm]{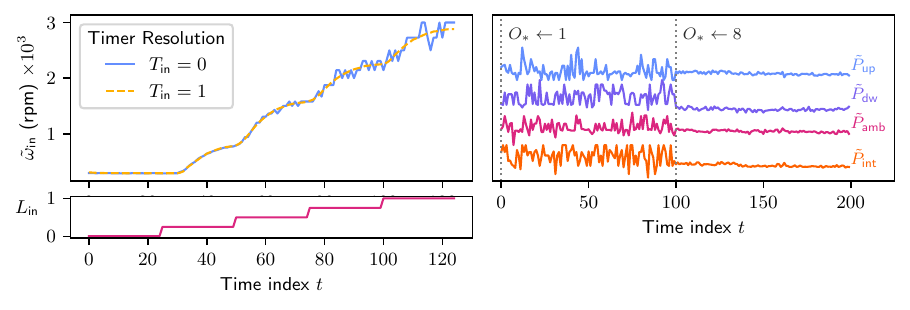}
}
\caption{(left) Measurements of the fan speed $\tilde{\omega}_\text{in}$ under a progressive increase in the fan load $L_\text{in}$ for different resolutions of the tachometer timer, with $(T_\text{in} = 0)$ corresponding to milliseconds and $(T_\text{in} = 1)$ to microseconds. (right) Effect of the barometer oversampling rate ($O_\text{up}, O_\text{dw}, O_\text{amb}, O_\text{int}$) on the resulting measurement ($\tilde{P}_\text{up}, \tilde{P}_\text{dw}, \tilde{P}_\text{amb}, \tilde{P}_\text{int}$). For all barometers, the oversampling rate is increased from 1 to 8 at $t=100$, while keeping all other chamber actuators and sensor parameters constant. The data and experimental setup for both plots can be found under the \nolinkurl{tach_resolution} and \nolinkurl{osr_barometers} experiments in the \href{https://github.com/juangamella/causal-chamber/tree/main/datasets/wt_test_v1}{\nolinkurl{wt_test_v1}} dataset.}
\label{fig:tachometer_oversampling}
\end{figure}%

\paragraph{$L_\textnormal{in} \to \{\tilde{C}_\textnormal{in}, \tilde{\omega}_\textnormal{in}\}, L_\textnormal{out} \to \{\tilde{C}_\textnormal{out}, \tilde{\omega}_\textnormal{out}\}$} The fan loads ($L_\textnormal{in}, L_\textnormal{out}$) define the duty cycle of the control signal sent to the fans, affecting their speed ($\tilde{\omega}_\textnormal{in}, \tilde{\omega}_\textnormal{out}$) and the current they consume ($\tilde{C}_\textnormal{in}, \tilde{C}_\textnormal{out}$). We provide mechanistic models describing these effects in \cref{ss:models_loads_fans}, comparing their outputs to real measurements in \autoref{fig:wt_models}. When the load is set to zero, the fan is completely powered off and no longer produces a tachometer signal (no corresponding observation shown in \autoref{fig:wt_models}); the resulting speed measurement corresponds to the last measured speed (see \autoref{fig:zero_load}).

\paragraph{$\{L_\textnormal{in}, L_\textnormal{out}, H\} \to \{\tilde{\omega}_\textnormal{in}, \tilde{\omega}_\textnormal{out}\}$} Each fan drives the flow of air through the wind tunnel, in turn making it easier or harder for the other fan to rotate. Thus, the speeds of the fans are coupled, and the strength of the coupling is affected by the hatch position (see \autoref{fig:loads_hatch_effects}, center).

\paragraph{$L_\textnormal{in} \to \tilde{C}_\textnormal{out}, L_\textnormal{out} \to \tilde{C}_\textnormal{in}$} Because they share the same power supply, the load of one fan has a small effect on the current drawn by the other (see \autoref{fig:loads_hatch_effects}, right). We believe these edges would not be present if each fan were driven by an independent power supply.

\paragraph{$T_\textnormal{in}\to\tilde{\omega}_\textnormal{in}, T_\textnormal{out}\to\tilde{\omega}_\textnormal{out}$} Changing the resolution ($T_\textnormal{in}, T_\textnormal{out}$) of the timers used in the fan tachometers also changes the resolution of the resulting speed measurement ($\tilde{\omega}_\textnormal{in}, \tilde{\omega}_\textnormal{out}$). Using a resolution of microseconds (\eg $T_\textnormal{in} = 1$) allows measuring smaller changes in the fan speed, with the difference being more noticeable at higher speeds (see \autoref{fig:tachometer_oversampling}, left).

\paragraph{$\{L_\textnormal{in}, L_\textnormal{out}, H\} \to \{\tilde{P}_\textnormal{up}, \tilde{P}_\textnormal{dw}\}$} By controlling the speed of the fans, the loads ($L_\textnormal{in}, L_\textnormal{out}$) affect the air pressure inside the wind tunnel, and thus the measurements of the two inner barometers ($\tilde{P}_\textnormal{up}, \tilde{P}_\textnormal{dw}$). By controlling the size of the additional opening to the outside, the hatch position ($H$) also affects the pressure inside the tunnel (\autoref{fig:loads_hatch_effects}). We provide mechanistic models describing these effects in \cref{s:models_wt}.

\paragraph{$\{L_\textnormal{in}, L_\textnormal{out}, H\} \to \tilde{P}_\textnormal{int}$} The fan loads and hatch position also affect the airflow and air pressure at the intake of the tunnel.

\paragraph{$O_\textnormal{up}\to\tilde{P}_\textnormal{up}, O_\textnormal{dw}\to\tilde{P}_\textnormal{dw}, O_\textnormal{amb}\to\tilde{P}_\textnormal{amb}, O_\textnormal{int}\to\tilde{P}_\textnormal{int}$} The oversampling rate determines how many barometric readings are averaged to produce a single measurement of the air pressure. A higher oversampling rate increases the precision of the barometers, reducing the noise in their measurements (see \autoref{fig:tachometer_oversampling}, right).

\paragraph{$\{R_\textnormal{out}, O_\textnormal{out}\}\to\tilde{C}_\textnormal{out}, \{R_\textnormal{in}, O_\textnormal{in}\}\to\tilde{C}_\textnormal{in}$} The current sensors encode their reading of the current as a voltage between 0 and 5 volts. This voltage is then read by the onboard computer, linearly mapping the range of $[0,R_\text{in}]$ ($[0,R_\text{out}]$) volts to $[0,1023]$. Thus, reducing the reference voltages $R_\text{in}, R_\text{out}$ increases the resolution of the current measurements (see \autoref{fig:voltage_sensors}, right), but can also cause the measurements to saturate if the voltage surpasses the reference voltage. As for the barometers, the oversampling rate ($O_\text{in}, O_\text{out}$) determines the number of readings that are averaged to produce a single measurement, thus affecting the precision of the measurements (see \autoref{fig:voltage_sensors}, left).

\paragraph{$A_1\to\{\tilde{S}_1, \tilde{S}_2\}, A_2\to\tilde{S}_2$} The signal fed to the speaker amplification circuit (\autoref{fig:chambers_diagram}c) is binary white noise with a period of $40 \mu s$: after every period the voltage is set to $0$ or $5$ volts with equal probability and is statistically independent of previous voltages. The two potentiometers act as controllable voltage dividers, affecting the amplitude of the signal at different points of the circuit. Thus, the first potentiometer ($A_1$) affects the measurements $\tilde{S}_1, \tilde{S}_2$, whereas the second potentiometer ($A_2$) only has an effect on the measurement $\tilde{S}_2$ (\autoref{fig:potentiometer_effects}).

\paragraph{$\{R_1, O_1\} \to\tilde{S}_1, \{R_2, O_2\} \to\tilde{S}_2$}  As for the current measurements, the reference voltages $R_1, R_2$ and oversampling rates $O_1, O_2$ have an effect on the resulting measurements $\tilde{S}_1, \tilde{S}_2$.

\paragraph{$A_1\to\tilde{M}$} The setting of the first potentiometer ($A_1$) controls the amplitude of the signal that is sent to the speaker, affecting the sound-level measurement $\tilde{M}$ taken by the microphone (\autoref{fig:potentiometer_effects}, left).

\paragraph{$\{L_\textnormal{in}, L_\textnormal{out}, H\}\to\tilde{M}$} The speed of the fans, controlled by the loads $L_\textnormal{in}, L_\textnormal{out}$, have an effect on the sound level $\tilde{M}$ measured by the microphone (see \autoref{fig:microphone_effects}, left). The position of the hatch ($H$) affects the amount of air flowing through the exhaust and over the microphone, also affecting its reading (\autoref{fig:microphone_effects}, center).

\paragraph{$\{R_M, O_M\} \to\tilde{M}$} As for the other sensors, the reference voltage affects the resolution of the microphone signal measurement and can induce saturation if the signal voltage rises above it. The oversampling rate $O_M$ determines the Nyquist frequency of the system and can introduce aliasing in the resulting signal (\autoref{fig:microphone_effects}, right).

\begin{figure}[H]
\centerline{
\includegraphics[width=179mm]{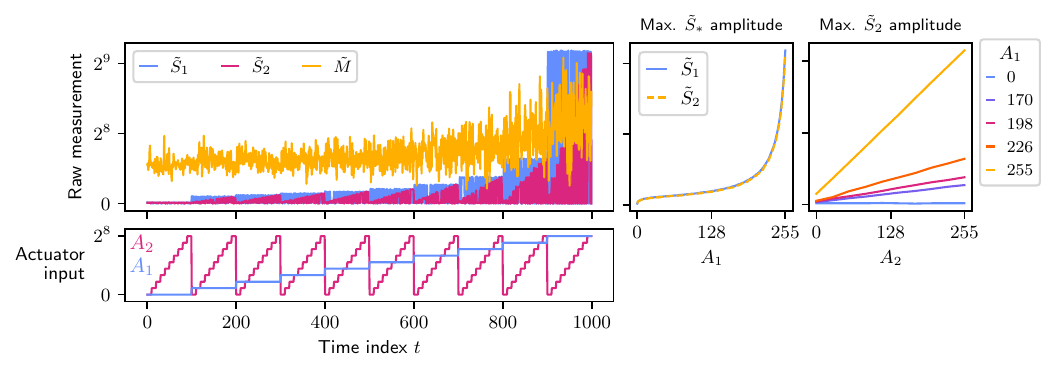}
}
\caption{(left) Time-series data of $\tilde{S}_1, \tilde{S}_2$ and the microphone output $\tilde{M}$, collected under the inputs $A_1, A_2$ to the amplification circuit potentiometers (bottom left), while all other actuators and sensor parameters are kept constant. The setting of the first potentiometer $A_1$ determines the maximum amplitude of the signals $\tilde{S}_1, \tilde{S}_2$ (center) and of the signal sent to the speaker, affecting the microphone output $\tilde{M}$ (left, gold). The setting of the second potentiometer $A_2$ only affects the maximum amplitude of $\tilde{S}_2$ (left), and the effect is linear for a fixed value of $A_1$ (right). The data and experimental setup for the plots can be found under the \nolinkurl{potis_coarse} and \nolinkurl{potis_fine} experiments in the \href{https://github.com/juangamella/causal-chamber/tree/main/datasets/wt_test_v1}{\nolinkurl{wt_test_v1}} dataset.}
\label{fig:potentiometer_effects}
\end{figure}%

\paragraph{$\tilde{P}_\text{dw} \to \{L_\text{in}, L_\text{out}\}$} 
In the \emph{pressure-control} configuration of the wind tunnel, the fan loads $L_\text{in}, L_\text{out}$ are set following a control mechanism to keep the pressure measured at the downwind barometer ($\tilde{P}_\text{dw}$) constant. In particular, the loads of the fans are set as
$$
L_\text{in} \leftarrow \begin{cases}
\min(1, u(t)) &\text{ if } u(t) > 0\\
0 &\text{ otherwise}
\end{cases} \quad \text{and} \quad L_\text{out} \leftarrow \begin{cases}
\min(1, -u(t)) &\text{ if } u(t) < 0\\
0 &\text{ otherwise}
\end{cases}
$$
where $u(t)$ is the controller output of the PID controller given by
$$u(t) := K_pe(t) + K_i\sum_{\tau=0}^te(t) + K_d(e(t) - e(t-1)),$$
where $e(t) := T - \tilde{P}_\text{dw}^t$ is the control error, \ie the difference in the pressure target $T$ and the measurement $\tilde{P}_\text{dw}^t$ of the downwind barometer at time point $t$. For the \href{https://github.com/juangamella/causal-chamber/tree/main/datasets/wt_pressure_control_v1}{\nolinkurl{wt_pressure_control_v1}} dataset, we set $(K_p, K_i, K_d) := (0.5, 0.1, 10^{-3})$ and set as pressure target $T$ the first measurement $\tilde{P}_\text{dw}^0$ taken by the downwind barometer after the chamber powers up. The control mechanism is executed internally by the chamber computer, producing a new output $u(t)$ at every time step. The related variables, \ie target $T$, control constants $K_p, K_i, K_d$, output $u(t)$ and error terms $e(t), \sum_{\tau=0}^te(t)$ and $(e(t) - e(t-1))$, are returned with each measurement as additional columns in datasets collected from the pressure-control configuration.

\begin{figure}[H]
\centerline{
\includegraphics[width=180mm]{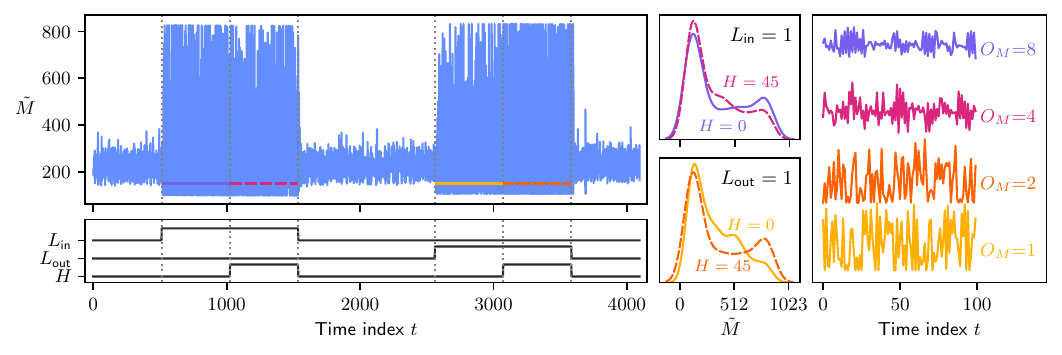}
}
\caption{(left) Time-series data of the microphone output $\tilde{M}$ collected under varying inputs to the fan loads $L_\text{in}, L_\text{out}$ and the hatch position $H$ (bottom left). All three actuators have an effect on the microphone output. (center) Marginal distribution of the microphone output $\tilde{M}$ for $L_\text{in}=1, L_\text{out}=0$ (top) and $L_\text{in}=0, L_\text{out}=1$ (bottom), corresponding to the regions marked in the time-series plot on the left. The hatch affects the amount of air that flows through the exhaust of the tunnel and over the microphone, affecting its readings. The effect of opening the hatch ($H=45$) is dependent on the loads of the fans, decreasing the airflow when $L_\text{in}=1, L_\text{out}=0$ and increasing it when $L_\text{in}=0, L_\text{out}=1$. (right) The oversampling rate determines the Nyquist frequency of the microphone system, introducing aliasing in the resulting signal. The data and experimental setup for the plots can be found in the \nolinkurl{mic_effects} experiment in the \href{https://github.com/juangamella/causal-chamber/tree/main/datasets/wt_test_v1}{\nolinkurl{wt_test_v1}} dataset.}
\label{fig:microphone_effects}
\end{figure}%

\begin{figure}[h]
\centerline{
\includegraphics[width=86mm]{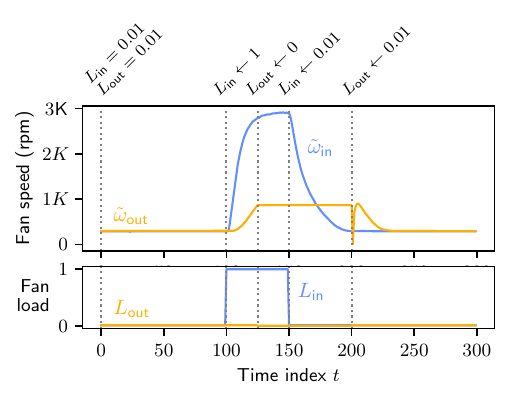}
}
\caption{By setting a fan load to zero (\eg $L_\text{out}\leftarrow 0$ at $t=125$), the fan is completely powered off and will decelerate until it stops rotating; however, at zero load, the fan no longer produces a tachometer signal, and the resulting speed measurement $\tilde{\omega}_\text{out}$ corresponds to the last measured speed (\eg $t \in [125,200]$). When powered up again ($t=200$) the fan always draws full power for an instant, quickly accelerating before returning to the level specified by the load. The data and experimental setup for the plot can be found under the \nolinkurl{zero_load} experiment in the \href{https://github.com/juangamella/causal-chamber/tree/main/datasets/wt_test_v1}{\nolinkurl{wt_test_v1}} dataset.}
\label{fig:zero_load}
\end{figure}%
\vfill

\section{Light Tunnel}
\label{s:lt_ground_truth}

\begin{figure}[H]
\centering
\includegraphics[width=167mm]{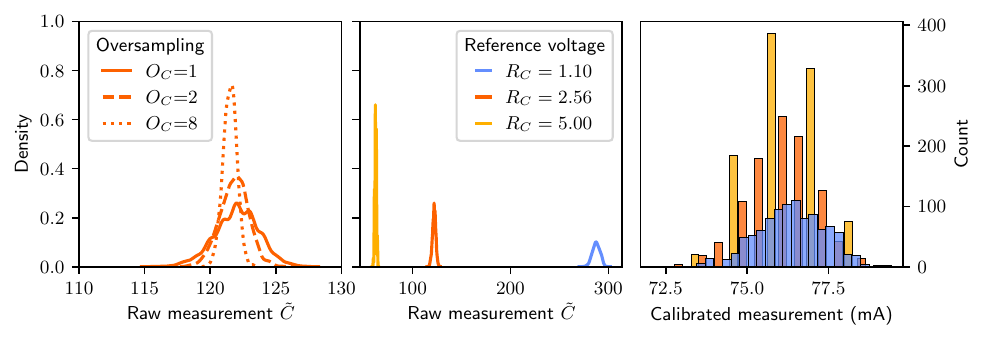}
\caption{Effect of the reference voltage $R_C$ and oversampling rate $O_C$ on the measurements $\tilde{C}$ of the current drawn by light source, when its brightness ($R,G,B$) is kept constant. The oversampling rate determines the amount of readings that are averaged to produce a single measurement, affecting the precision of the sensor (left). For a constant oversampling rate ($O_C=1$), changing the reference voltage results in a shift of the sensor output (middle) and an increase in the resolution of the calibrated measurement (right). While the plots show the current measurement $\tilde{C}$ in the light-tunnel, the same principles apply for the other sensors which encode their measurement as voltage ($\tilde{\theta}_1, \tilde{\theta}_2, \tilde{M}, \tilde{S}_1, \tilde{S}_2$) and for the wind-tunnel barometers ($\tilde{P}_\text{dw}, \tilde{P}_\text{up},\tilde{P}_\text{amb}, \tilde{P}_\text{int}$), which also allow for varying oversampling rates  ($O_\text{dw}, O_\text{up},O_\text{amb}, O_\text{int}$). The data corresponds to the \nolinkurl{current_sensor} experiment from the \href{https://github.com/juangamella/causal-chamber/tree/main/datasets/lt_test_v1}{\nolinkurl{lt_test_v1}} dataset.}
\label{fig:voltage_sensors}
\end{figure}

\begin{figure}[H]
\centering
\includegraphics[width=157mm]{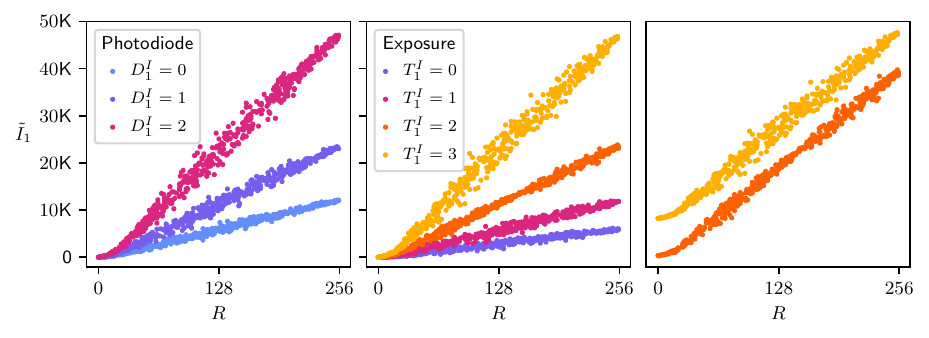}
\caption{Effect of the photodiode size ($D^I_1$) and exposure time ($T^I_1$) on the infrared measurements ($\tilde{I}_1$) of the light sensor  nearest the light source. We display the output $\tilde{I}_1$ of the sensor for varying brightness of the red LEDs on the light source ($R$), while keeping $G=B=0$ and all other actuators constant. Increasing the photodiode size (left) or exposure time (center) both increase the sensitivity of the sensor. Due to an interplay between the exposure time and the pulse-width-modulation frequency of the light source, changes in the exposure time also affect the conditional distribution of the measurements given the light-source intensity ($R,G,B$). For example, when comparing the measurements for $T^I_1=2$ and $T^I_1=3$ in the rightmost panel (same data as in the center panel, but vertically scaled for comparison), we can see a difference in the variance as a function of $R$. The effects on the visible-light measurements $\tilde{V}_1$ follow the same principles and are not shown. The data corresponds to the \nolinkurl{ir_sensors} experiment from the \href{https://github.com/juangamella/causal-chamber/tree/main/datasets/lt_test_v1}{\nolinkurl{lt_test_v1}} dataset.}
\label{fig:diodes_exposures}
\end{figure}

\paragraph{$\{R,G,B\} \to \{\tilde{I}_1, \tilde{I}_2, \tilde{I}_3, \tilde{V}_1, \tilde{V}_2, \tilde{V}_3\}$}
The brightness settings of the light-source colors $(R,G,B)$ affect the readings of all light-intensity sensors. The effect of each color is approximately linear with heteroscedastic noise (\autoref{fig:chambers_data}d), with the slope determined by its typical wavelength and the spectral sensitivity of the sensor (see \nolinkurl{light_source} and \nolinkurl{light_sensor} components in \autoref{apx:datasheets}).

\paragraph{$\{R,G,B\} \to \tilde{C}$} The brightness settings also affect the electric current drawn by the light source. The effect is again approximately linear with roughly the same slope for each channel (see \autoref{fig:chambers_data}d).

\begin{figure}[H]
\centering
\includegraphics[width=140mm]{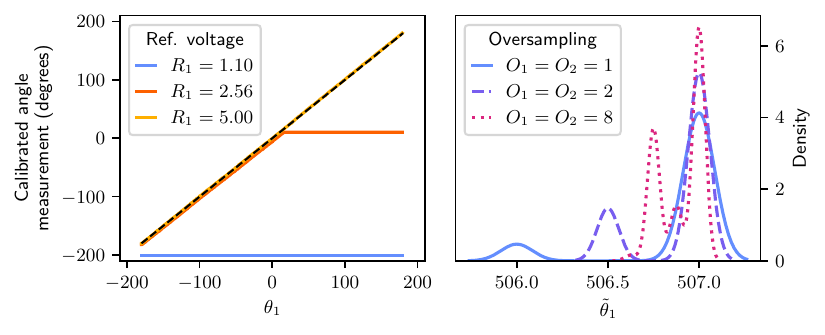}
\caption{Effect of the reference voltage ($R_1$) and oversampling rate ($O_1$) on the angle measurements of the first polarizer ($\tilde{\theta}_1$). (left) We show the calibrated measurements of the angle sensor (see \autoref{tab:light_tunnel_vars} for details) for different values of the motor setting $\theta_1$ and reference voltages $R_1$, while keeping the oversampling rate constant at $O_1=1$. Lowering the reference voltage $R_1$ results in a saturation of the sensor, being completely saturated (\ie returning a constant reading) at $R_1=1.1$. (right) Uncalibrated measurements $\tilde{\theta}_1$ for a constant motor setting $\theta_1 = 0$ and reference voltage $R_1=5$. The stochasticity of the measurement is due to electrical noise in the voltage measurement---increasing the oversampling rate $O_1$ increases the precision of the sensor. The same principles apply to the measurements $\tilde{\theta}_2$ of the second angle sensor. The data corresponds to the \nolinkurl{angle_sensors} experiment from the \href{https://github.com/juangamella/causal-chamber/tree/main/datasets/lt_test_v1}{\nolinkurl{lt_test_v1}} dataset.}
\label{fig:angle_sensors}
\end{figure}

\paragraph{$\{R_C, O_C\} \to \tilde{C}$} As for the wind-tunnel, the sensor measuring the current drawn by the light source encodes its reading into a voltage between 0 and 5 volts. This voltage is then read by the onboard computer, linearly mapping the range of $[0,R_C]$ volts to $[0,1023]$. Thus, reducing the reference voltage $R_C$ increases the resolution of the current measurements (see \autoref{fig:voltage_sensors}, right). The oversampling rate ($O_C$) determines the number of readings that are averaged to produce a single measurement, affecting its precision (\autoref{fig:voltage_sensors}, left).

\paragraph{$\{D^I_j, T^I_j\} \to \tilde{I}_j, \{D^V_j, T^V_j\} \to \tilde{V}_j$} The measurements of each light sensor are affected by the choice of photodiode ($D^I_j, D^V_j$) and gain ($T^I_j, T^V_j$) used to perform the infrared ($\tilde{I}_j$) and visible-light ($\tilde{V}_j$) readings. Larger photodiodes collect light over a larger area, increasing the sensitivity of the sensor (\autoref{fig:diodes_exposures}, left). Higher exposure times achieve a similar effect by increasing the reading duration, which also has an effect on the conditional distribution of the measurements given the light source brightness (\autoref{fig:diodes_exposures}, right). A timing mechanism ensures that changes in the exposure time do not affect the overall measurement time.

\paragraph{$\{\theta_j, R_j, O_j\} \to \tilde{\theta}_j$} The motor settings $\theta_1, \theta_2$ determine the positions of the polarizer frames. The positions are encoded into a voltage by means of rotary potentiometers, resulting in the measurements $\tilde{\theta}_1, \tilde{\theta}_2$. Changes in the reference voltages $R_1, R_2$ affect the resolution of the measurements, but can also cause the sensor to saturate (see \autoref{fig:angle_sensors}, left). The oversampling rates ($O_1, O_2$) determine the number of readings that are averaged to produce a single measurement, affecting its precision (\autoref{fig:angle_sensors}, right).

\paragraph{$\{\theta_1, \theta_2\} \to \{\tilde{I}_3, \tilde{V}_3, \tilde{\textmd{I}}\textmd{m}\}$} The position of the polarizers affects the intensity and spectral composition (\ie color) of the light passing through them, affecting the readings ($\tilde{I}_3, \tilde{V}_3$) of the third light sensor (see \autoref{fig:chambers_data}c and \autoref{fig:polarizer_model}) and the images ($\tilde{\textmd{I}}\textmd{m}$) captured by the camera (see \autoref{fig:chambers_data}e). The effect is described by Malus' law---see \cref{ss:models_malus} for more details.

\paragraph{$\{\textmd{Ap}, \textmd{ISO}, T\} \to \tilde{\textmd{I}}\textmd{m}$} The camera parameters affect the brightness of the captured images while introducing a variety of side effects (\autoref{fig:chambers_data}e). For example, a higher gain (ISO) increases the noise, while the aperture (Ap) changes the depth of field, introducing blur.

\paragraph{$\{L_{j1}, L_{j1}\} \to \tilde{I}_j, \tilde{V}_j$} Besides the light source, two additional LEDs placed by each light sensor (see \autoref{fig:chambers_diagram}d) have an effect on its reading. To avoid affecting the measurements of the other light sensors, the LEDs only turn on when their corresponding sensor is taking a measurement. Their brightness is controlled by means of digital rheostats that control the current flowing to each LED. The settings $L_{j1}, L_{j2}$ correspond to the wiper position of each rheostat. The relationship between the setting and the intensity readings follows an exponential function (\autoref{fig:chambers_data}b), resulting from the LED's voltage-current characteristic, luminosity response and typical wavelength, and the spectral sensitivity of the sensor. The brightness of the LEDs is much lower than that of the light source (\autoref{fig:leds_conditioning}). See the \nolinkurl{led} and \nolinkurl{light_sensor} components in \autoref{apx:datasheets} for the corresponding datasheets.
\\

\begin{figure}[H]
\centering
\includegraphics[width=140mm]{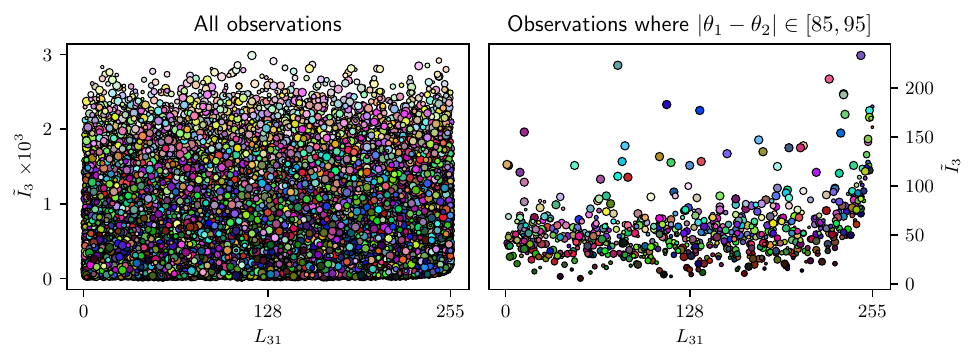}
\caption{Effect of the LED brightness settings ($L_{31}, L_{32}$) on the reading $\tilde{I}_3$ of the third light sensor, which is separated from the light source by both polarizers. The light source colors $R,G,B$ and brightness settings $L_{11}, \ldots L_{32}$ are sampled independently and uniformly at random from the range $\{0,\ldots,255\}$, and the polarizer angles $\theta_1, \theta_2$ from $[-15,105]$. In both plots, the color corresponds to $R,G,B$, and the size to $L_{32}$. Because the brightness of the LEDs is very small when compared to that of the light source, their effect on the intensity reading $\tilde{I}_3$ is barely noticeable when looking at all observations (left), but becomes apparent, for example, when the polarizers are close to orthogonal (right) and block most of the light emanating from the light source. The data corresponds to the \nolinkurl{actuators_white} experiment from the \href{https://github.com/juangamella/causal-chamber/tree/main/datasets/lt_walks_v1}{\nolinkurl{lt_walks_v1}} dataset.}
\label{fig:leds_conditioning}
\end{figure}

%%%%%%%%%%%%%%%%%%%%%%%%%%%%%%%%%%%%%%%%%%%%%%%%%%%%%%%%%%%%%%%%%%%%%%%%%%%%%%%%
%% APPENDIX
\newpage
\part{Mechanistic Models of The Chambers}
\label{apx:mechanistic_models}

\setcounter{section}{0}

This appendix provides mechanistic models describing several effects and processes in the causal chambers. Each model is labeled with a letter and a number (\eg B1, C2); different letters correspond to different modeled quantities, and higher numbers mean increased fidelity. A Python implementation of each model can be found in \href{https://github.com/juangamella/causal-chamber}{\nolinkurl{github.com/juangamella/causal-chamber}}, and is accessible through the \nolinkurl{causalchamber} package; examples and the code to generate the plots in this section can be found in the Jupyter notebook \href{https://github.com/juangamella/causal-chamber-paper/blob/main/case_studies/mechanistic_models.ipynb}{\nolinkurl{case_studies/mechanistic_models.ipynb}} in the paper repository \href{https://github.com/juangamella/causal-chamber-paper}{\nolinkurl{github.com/juangamella/causal-chamber-paper}}.

Excluding models D1 and F1--F3, the models in this section are agnostic to phenomena arising from the sensors that measure the modeled physical quantities, such as measurement and quantization noise, saturation, or non-linearities in their response. For this reason, we denote the modeled quantities without a tilde (\eg $\omega_\text{in}, P_\text{dw}, I_3)$ to differentiate them from the corresponding sensor measurements (\eg $\tilde{\omega}_\text{in}, \tilde{P}_\text{dw}, \tilde{I}_\text{3})$.

%%%%%%%%%%%%%%%%%%%%%%%%%%%%%5
%% SECTION
\section{Wind Tunnel}
\label{s:models_wt}

We provide models describing the effect of the fan load on the speed of the fans and the drawn current (\cref{ss:models_loads_fans}); the effect of the fan loads and hatch position on the reading of the downwind barometer (\cref{ss:models_downwind}); and the difference between the readings of the up- and downwind barometers, as described by Bernoulli's principle (\cref{ss:models_bernoulli}).

\begin{figure}[h]
\centerline{
\includegraphics[width=180mm]{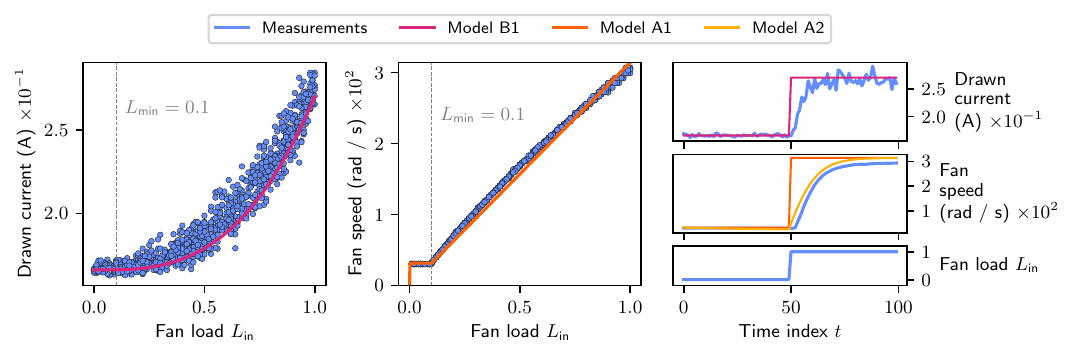}
}
\caption{Comparing model predictions with real measurements. (left and center): models 
B1 and A1 of the drawn current and speed of the fans, respectively, as a function of the load $L_\text{in}$, compared to steady-state measurements from the \href{https://github.com/juangamella/causal-chamber/tree/main/datasets/wt_bernoulli_v1}{\nolinkurl{wt_bernoulli_v1}} dataset. When the load is set to zero, the fan is completely powered off and no longer produces a tachometer signal (not shown)---see also \autoref{fig:zero_load}.
(right) Output of the models B1, A1 and A2 for time-series data from the \href{https://github.com/juangamella/causal-chamber/tree/main/datasets/wt_test_v1}{\nolinkurl{wt_test_v1}} dataset. The parameters of all models are set as in \cref{ss:models_loads_fans}.}
\label{fig:wt_models}
\end{figure}%

The fans used in the wind tunnel are high-speed fans for industrial cooling applications. Their speed is controlled via pulse-width-modulation (PWM), where the pulse-width, or duty cycle, corresponds to the fan loads $L_\text{in}$ and $L_\text{out}$ of the wind tunnel. Additional details and technical specifications of the fans can be found in the datasheets of \autoref{apx:datasheets}, under the \nolinkurl{fan} component. As we derive the different models of this section, we will refer to these datasheets for the values of certain model parameters.

\subsection{Effect of the fan load on fan speed and drawn current}
\label{ss:models_loads_fans}

Two important aspects of the fans' design are used throughout the models of the fan speed (A1, A2) and drawn current (B1): first, the fans are designed so their steady-state speed scales broadly linearly with the load. At full load, the fans turn at a maximum speed $\omega_\text{max}$, which is specified by the manufacturer as $314.16 \text{ rad}/\text{s}$ (3000 RPM). Second, due to their intended application, unless completely powered off (\ie $L_\text{in},L_\text{out} = 0$) the fans never operate below a certain speed, corresponding to a minimum effective load which we denote by $L_\text{min}$. This value is specified by the manufacturer to be around $0.2$, but our experiments show it to be closer to $0.1$ (see \autoref{fig:wt_models}, center).
    
\subsubsection*{Model A1}

We model the steady-state speed $\omega$ of the fan under a load $L$ as
\begin{align}
  \omega = \begin{cases}
      \max(L, L_\text{min})\omega_\text{max} &\text{if } L > 0\\
      0 &\text{if } L = 0
  \end{cases}
\end{align}
where $\Theta := (\omega_\text{max}, L_\text{min})$ are the parameters of the model. The first two correspond to the maximum speed of the fan and the minimum effective load. From the datasheets and our experiments, we can set $L_\text{min}= 0.1$ and $\omega_\text{max} = 314.16 \text{ rad}/\text{s}$ (see \autoref{fig:wt_models}, left).

\subsubsection*{Model B1} We model the effect of fan load on its drawn current through the affinity laws \parencite[\eg][Sec. 9.3.5]{nakayama2018introduction}, which state that the power $W$ consumed by a fan is proportional to the cube of its speed $\omega$, that is, $W = W_\text{max} (\omega / \omega_\text{max})^3$, where $W_\text{max}$ is the power consumed by the fan at its maximum speed $\omega_\text{max}$. Because the fans operate at a constant voltage, we can rewrite the above expression in terms of the drawn current $C$ as $C = C_\text{max} (\omega / \omega_\text{max})^3$, where $C_\text{max}$ is the current drawn at maximum speed. Accounting for the no-load current $C_\text{min}$ of the motor, and combining the above with model A1 of the fan speed, we arrive at the model
\begin{align}
  C = \begin{cases}
      C_\text{min} + \max(L_\text{min},L)^3 (C_\text{max} - C_\text{min}) &\text{if } L > 0\\
      C_\text{min} &\text{if } L = 0
  \end{cases}
\end{align}
where $\Theta := (L_\text{min}, C_\text{min}, C_\text{max})$ are the parameters of the model. As before, we can set $L_\text{min}=0.1$; from the datasheet, we can set the nominal current of the fan to $C_\text{max}=0.26$ A, and empirically determine the no-load current to be approximately $C_\text{min}=0.166$ A.

\subsubsection*{Model A2}

To model the dynamics of the fan speed beyond the steady state, we express the change in the speed through the torque-balance equation
\begin{align}
    d\omega = \frac{1}{I} (\tau(L) - K \omega^2)dt,
\end{align}
where $I$ is the fan's moment of inertia, $\tau(L)$ is the torque applied by the fan motor for a given load $L$, and $K\omega^2$ is the opposing torque due to the drag experienced by the fan. Here, $K$ is a constant that depends on the geometry of the blades and the density of the air \parencite[\eg][Sec. 9.3.5]{nakayama2018introduction}. These constitute the parameters $\Theta := (I, \tau(L), K)$ of our model, which we approximate as follows:

\begin{itemize}
    \item We model the fan as a solid cylinder, resulting in a moment of inertia $I = \frac{1}{2}mr^2$, where $m,r$ are the mass and radius of the disc, respectively. The fan has a radius of $r=0.059$m, and we estimate its mass as $m=0.02$kg, resulting in $I = 3.48 \times 10^{-5} \text{kg}/\text{m}^2$.
    \item We approximate the load-to-torque function as        
    $$\tau(L) := \begin{cases}
        T(\max(L_\text{min},L)^3 (C_\text{max} - C_\text{min})) &\text{if } L > 0\\
        0 &\text{if } L = 0
    \end{cases}$$    
    where $T$ is the torque constant of the motor and the remaining terms correspond to the drawn current as described by model B1 (after substracting the no-load current $C_\text{min}$). In \autoref{fig:wt_models} and \autoref{fig:benchmarks_2}f, we set $T = 0.05 \text{Nm} / \text{A}$.
    \item To obtain a value for the constant $K$ in our model, we solve the steady-state equation $d\omega = \frac{1}{I}(\tau(1) - K \omega_\text{max}^2) = 0$ with $I$ and $\tau(L)$ as above, resulting in $K=5.26 \times 10^{-8}$.
\end{itemize}
In \autoref{fig:wt_models} we compare the output of the models A1 and A2 to real measurements collected from the chamber. An important source of misspecification is that we model the intake and exhaust fans independently, whereas in the tunnel, their speeds affect each other (see \autoref{fig:benchmarks_2}f).

\subsection{Effect of the fan loads and hatch position on the downwind barometer}
\label{ss:models_downwind}

In this section, we provide models describing the effect of the fan loads and hatch position on the static pressure inside the wind tunnel, corresponding to the measurement $\tilde{P}_\text{dw}$ of the downwind barometer. Models C1--C3 relate the pressure to the fan speeds $\omega_\text{in}$ and hatch position $H$. We can combine them with models A1 or A2 of the fan speeds to simulate the barometer reading as a function of the fan loads ($L_\text{in}, L_\text{out}$), as we do in \autoref{fig:benchmarks_2}f.

To model the complex dynamics of the airflow through the wind tunnel, models C1--C3 make some simplifying assumptions. In first place, we assume that the change in static pressure can be simply computed as the difference of the static pressure produced by each fan. Furthermore, we model each fan independently, ignoring mutual effects between them.

\subsubsection*{Model C1} As a first approach, we model the static pressure $P_\text{dw}$ inside the tunnel as
\begin{align}
\label{eq:model_1_pressure}
    P_\text{dw} = P_\text{amb} + S_\text{max}\left(\frac{\omega_\text{in}}{\omega_\text{max}}\right)^2 - S_\text{max}\left(\frac{\omega_\text{out}}{\omega_\text{max}}\right)^2,
\end{align}
where $P_\text{amb}$ is the static pressure outside the wind tunnel (corresponding to the barometer reading $\tilde{P}_\text{amb}$). The other two terms on the RHS correspond to the static pressures produced by the intake and exhaust fans, resulting from the the affinity laws \parencite[\eg][chapter 5.14]{leckner2008ludwig}. The parameters of our model are $\Theta := (S_\text{max}, \omega_\text{max})$, corresponding to the maximum static pressure and maximum speed of the fan. From the technical specifications, we can set these values to $74.82$ Pa and $314.16 \text{ rad}/\text{s}$, respectively.

\subsubsection*{Model C2} Model C1 assumes that the fans always produce their maximum static pressure $S_\text{max}\left(\frac{\omega}{\omega_\text{max}}\right)^2$ at a given speed $\omega$. In reality, there is a trade-off between the airflow and static pressure produced by the fan, dictated by the system's resistance to the flow of air, which is called \emph{impedance}. As an illustration, a fan blowing into an enclosure with no other openings produces no airflow but maximum static pressure, whereas a fan at the boundary of two completely open enclosures produces its maximum airflow but no static pressure \parencite{fans2024tang}. To account for this effect, we model the pressure $P$ as
$$
    P_\text{dw} = P_\text{amb} + S_Z(\omega_\text{in}) - S_Z(\omega_\text{out}),
$$
where $S_Z(\omega)$ is the static pressure produced by the fan at speed $\omega$ in a system with impedance $Z$. We compute $S_Z(\omega)$ as the intersection of the impedance curve $S=ZQ^2$ with the pressure-airflow characteristic (PQ-curve) of the fan at speed $\omega$ \parencite[see, \eg][]{fans2024tang}. Because the PQ-curve of our fans is not available, we approximate them as the linear relation
$$S = \left(\frac{\omega_\text{in}}{\omega_\text{max}}\right)^2S_\text{max} - \left(\frac{\omega_\text{in}}{\omega_\text{max}}\right)\frac{S_\text{max}}{Q_\text{max}}Q,$$
where $Q_\text{max}$ is the maximum airflow of the fan. The impedance $Z \in [0,\infty]$ is a result of the complex dynamics of the air as it flows through the wind tunnel, and we cannot directly calculate it. For a more intuitive interpretation, we can express it in terms of the ratio $r$ of the maximum airflow that the fans produce when turning at full speed, yielding $Z := \frac{S_\text{max}}{Q_\text{max}^2}\left(\frac{1-r}{r^2}\right)$ for $r \in [0,1]$. Under this parametrization, our model becomes
\begin{align}
\label{eq:model_pressure_2}
    P_\text{dw} = P_\text{amb} + S_r(\omega_\text{in}) - S_r(\omega_\text{out}),
\end{align}
with parameters $\Theta := (S_\text{max}, \omega_\text{max}, Q_\text{max}, r)$. Following the fan's specifications, we set the maximum airflow to $Q_\text{max} = 0.052 \text{m}^3/\text{s}$. For the simulation results in \autoref{fig:benchmarks_2}f, we set $r=0.7$.

\subsubsection*{Model C3} So far. we have not considered the effect of the hatch position ($H$) on the pressure inside the tunnel. Building on Model C2, we model the effect of the hatch as a change in the impedance of the system. Opening the hatch lowers the system impedance and, thus, the static pressure produced by the fans. As in \eqref{eq:model_pressure_2}, we model the pressure as
\begin{align}
    P_\text{dw} = P_\text{amb} + S_{r(H)}(\omega_\text{in}) - S_{r(H)}(\omega_\text{out}),
\end{align}
except now the maximum airflow ratio is given by the function $r(H) = \min(1, r_0 + \beta \frac{H}{45})$, where $H \in [0,45]$, $r_0 \in [0,1]$ is the ratio when the hatch is closed ($H=0$), and $\beta > 0$ is the effect of the hatch, which we model as linear. Thus, the parameters of the model become $\Theta := (S_\text{max}, \omega_\text{max}, Q_\text{max}, r_0, \beta)$. For the simulation results in \autoref{fig:benchmarks_2}f, we set $r_0=0.7$, $\beta=0.15$, and all other parameters as for the previous models.

\subsection{Relationship between the up- and downwind barometers}
\label{ss:models_bernoulli}

Due to their positioning, the pair of barometers inside the wind tunnel function as a pitot tube, where the upwind barometer ($\tilde{P}_\text{up}$) measures the stagnation or total pressure, and the downwind barometer ($\tilde{P}_\text{dw}$) measures the static pressure. Ignoring measurement noise, the readings of both barometers can be related by Bernoulli's principle:

\begin{align}
\label{eq:bernoulli}
    \tilde{P}_\text{up} = \tilde{P}_\text{dw} + \frac{1}{2}\rho v^2 + \Delta.  
\end{align}
Above, $\rho$ is the density of the air, $v$ is the speed of the airflow and $\Delta$ is the offset between the readings of the barometers due to manufacturing differences, which we empirically determine to be around $\Delta = 7.1$ Pa. By solving for $v$, we can use \eqref{eq:bernoulli} to estimate the speed of the airflow through the tunnel. Model D1 relates the difference in the barometer readings to the speed of the intake fan.

\subsubsection*{Model D1}

To express the speed $v$ of the airflow as a function of the intake fan speed, we again rely on the affinity laws, which state that the airflow $Q$ produced by the fan is proportional to its speed $\omega_\text{in}$, i.e.

$$Q = \left(\frac{\omega_\text{in}}{\omega_\text{max}}\right) Q_\text{max},$$
where $\omega_\text{max}, Q_\text{max}$ are the maximum speed and airflow of the fan. Dividing by the area $A$ of the fan opening, we obtain the following expression for airspeed
$$v = \frac{1}{A}\left(\frac{\omega_\text{in}}{\omega_\text{max}}\right) Q_\text{max}.$$
Combining this with \eqref{eq:bernoulli}, we arrive at

\begin{align}
\label{eq:barometers_bernoulli}
\tilde{P}_\text{up} - \tilde{P}_\text{dw}
= \frac{\rho}{2A^2}\left(\frac{Q_\text{max}}{\omega_\text{max}}\right)^2\omega_\text{in}^2 + \Delta,
\end{align}
where $\Theta := (\rho, A, \omega_\text{max}, Q_\text{max},\Delta)$ are the parameters of our model. Given the specifications of the fan, we have $A = \pi 0.06^2\text{ m}^2$ and $\omega_\text{max} = 314.16 \text{ rad}/\text{s}$. We can set $\rho=1.2$ and we empirically determine $\Delta=7.1$ Pa. The same difficulties apply to estimating $Q_\text{max}$ as for $S_\text{max}$ in model C1. Furthermore, in this model, we ignore the additional effect that the exhaust fan has on the total airflow through the tunnel. For the symbolic regression task in \autoref{fig:benchmarks_2}e, we turn off the exhaust fan and take steady-state measurements for different loads of the intake fan (see the \href{https://github.com/juangamella/causal-chamber/tree/main/datasets/wt_bernoulli_v1}{\nolinkurl{wt_bernoulli_v1}} dataset for more details).

%%%%%%%%%%%%%%%%%%%%%%%%%%%%%5
%% SECTION
\section{Light Tunnel}
\label{s:modelst_lt}

In \cref{ss:models_malus}, we provide some background on Malus' law and provide a model of the effect of the polarizer angles ($\theta_1, \theta_2$) on the light intensity reaching the third light sensor ($\tilde{I}_3, \tilde{V}_3)$. In \cref{ss:image_capture} we provide models of the image-capture process with increasing degrees of fidelity.

\subsection{Polarizer Effects}
\label{ss:models_malus}

\begin{figure}[ht]
\centerline{
\includegraphics[width=180mm]{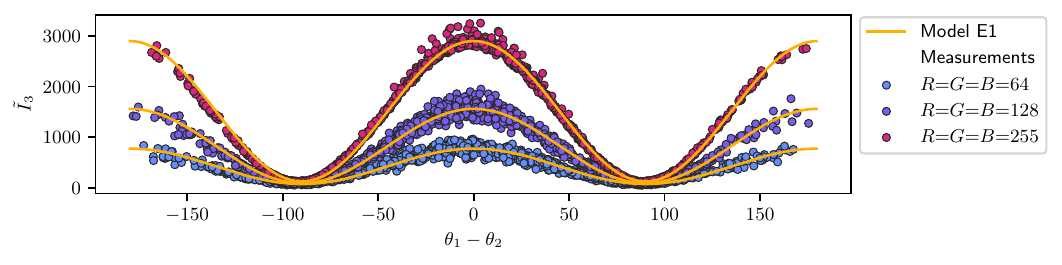}
}
\caption{Comparing the prediction of model E1 (Malus' law) with real measurements gathered from the light tunnel. Model E1 is fit to data gathered under different brightness settings of the light source (\ie $R=G=B=64,128,255$), obtaining a coefficient of determination of $R^2 = 0.93, 0.97, 0.99$, respectively. The data for $R=G=B=255$ is used in the symbolic regression task of \autoref{fig:benchmarks_2}e. The signal-to-noise decreases with the light-source brightness, allowing for more challenging scenarios. The data can be found in the \href{https://github.com/juangamella/causal-chamber/tree/main/datasets/lt_malus_v1}{\nolinkurl{lt_malus_v1}} dataset.}
\label{fig:polarizer_model}
\end{figure}%

Depending on their angle, the linear polarizers dim the light passing through them. Malus' law describes their effect on light intensity \parencite{collett2005field}: given a beam of totally polarized light and a perfect polarizer, the intensity $I$ after the polarizer is given by $I = I_0 \cos^2(\theta)$,
where $I_0$ is the intensity before the polarizer and $\theta$ is the angle between the polarizer axis and the polarization of the incident light. For unpolarized light passing through two polarizers at angles $\theta_1$ and $\theta_2$, the intensity after the polarizers becomes
$$I = \frac{1}{2}I_0\cos^2(\theta_1 - \theta_2).$$

In practice, polarizers are imperfect: they are not fully transparent to light polarized parallel to their polarization axis, and they do not block all light with polarization orthogonal to it. Furthermore, their effect depends on the frequency of light. For an unpolarized light beam with a given spectral composition, a pair of polarizers will scale its intensity by $T^p < 1/2$; when orthogonal, they will still allow some light to pass through, scaling the intensity by $T^c > 0$.

\subsubsection*{Model E1}

To model the effect of the polarizer angles $\theta_1,\theta_2$ on the intensity $I_3$ of light reaching the third sensor, we can reformulate the generalization of Malus' law for imperfect polarizers \parencite{lages2008composition} in terms of the transmission rates $T^p, T^c$, resulting in the model
$$I_3 = I_0 [(T^p - T^c)\cos^2(\theta_1 - \theta_2) + T^c],$$
where $\Theta := (I_0, T^p, T^c)$ are the parameters of the model. The radiant intensity before the polarizers ($I_0$) is given in watts per steradian (W/sr), and depends \eg on the brightness of the tunnel light source. For the polarizing film used in the light tunnel, values for the transmission rates $T^p,T^c \in [0,1]$ can be found in \autoref{apx:datasheets} under the \nolinkurl{polarizer} component. We can rewrite the above model as
$$I_3 = \beta_1 \cos^2(\theta_1 - \theta_2) + \beta_0,$$
becoming our ground truth for the symbolic regression task in \autoref{fig:benchmarks_2}e. We compare, in \autoref{fig:polarizer_model}, the output of model E1 to measurements gathered from the chamber, with the parameters $\beta_0, \beta_1$ fit to the data.

\subsection{Image Capture}
\label{ss:image_capture}

As models of the image-capture process in the light tunnel, we provide simple simulators that produce a synthetic image given the light-source setting $R,G,B$ and polarizer angles $\theta_1,\theta_2$. The resulting image simulates the image ($\tilde{\text{I}}$m) produced by the light tunnel and consists of a hexagon over a black background, whose color is given by the RGB vector $(\tilde{R}, \tilde{G}, \tilde{B}) \in [0,1]^3$ and models the response of the camera sensor. The models F1--F3 differ in how the vector $(\tilde{R}, \tilde{G}, \tilde{B})$ is computed. In \autoref{fig:benchmarks_2}f, we compare their output to real images produced by the light tunnel.

\subsubsection*{Model F1} The first model assumes the linear polarizers to be perfect and that their effect is uniform across all wavelengths. Furthermore, it assumes that the camera sensor is perfectly calibrated to the light source. The resulting color vector is given by
\begin{align}
    \begin{pmatrix}
        \tilde{R}\\
        \tilde{G}\\
        \tilde{B}\\
    \end{pmatrix}
    = \cos^2(\theta_1 - \theta_2)\frac{1}{255}
    \begin{pmatrix}
        R\\
        G\\
        B\\
    \end{pmatrix},
\end{align}
where $\cos^2(\theta_1 - \theta_2)$ models the dimming effect of the polarizers according to Malus' law (see \cref{ss:models_malus} above).

\subsubsection*{Model F2} The second model explicitly models the output of the camera sensor as a result of its spectral sensitivity and the white balance correction applied by the processor in the camera. Our complete model is
\begin{align}
    \begin{pmatrix}
        \tilde{R}\\
        \tilde{G}\\
        \tilde{B}\\
    \end{pmatrix}
    = \min \{1,eWS\cos^2(\theta_1 - \theta_2)\frac{1}{255}
    \begin{pmatrix}
        R\\
        G\\
        B\\
    \end{pmatrix}\},
\end{align}
where
\begin{itemize}
    \item $S \in \mathbb{R}^{3 \times 3}_+$ models the response of the sensor to the light of each color,
    \item $W := \diag{w_R, w_B, w_C}$ with $w_R, w_B, w_C \in \mathbb{R}_+ : w_R + w_B + w_C = 1$ is the white-balance correction applied by the camera, and
    \item $e > 0$ is an additional parameter to model the exposure of the sensor as affected by the camera parameters aperture (Ap), shutter speed ($T_\text{Im}$), and sensor gain (ISO).
\end{itemize}
The term $\min\{1,x\} := (\min \{1,x_i\})_{i=1,2,3}$ truncates the resulting color vector to the interval $[0,1]^3$, modeling overexposure of the camera sensor.

\subsubsection*{Model F3}

We separately model the effect of the polarizers on each of the frequencies produced by the light source. Our model becomes
\begin{align}
    \begin{pmatrix}
        \tilde{R}\\
        \tilde{G}\\
        \tilde{B}\\
    \end{pmatrix}
    = \min \{1,eWS    
    \diag{(T^p - T^c)
    \cos^2(\theta_1 - \theta_2)    
    + T^c}
    \frac{1}{255}
    \begin{pmatrix}
        R\\
        G\\
        B\\
    \end{pmatrix}\},
\end{align}
where $T^p := (T^p_R, T^p_G, T^p_B) \in [0,1]^3$ are the transmission rates for each color when the polarizers are aligned, and $T^c \in [0,1]^3$ when their polarization axes are perpendicular. From the technical specifications of the light source and polarizers, appropriate values are $T^p = (0.29,0.35,0.33)$ and $T^c = (0.02, 0.08, 0.18)$; for the corresponding datasheets, see the \nolinkurl{light_source} and \nolinkurl{polarizer} components in \autoref{tab:datasheets}.
\vfill

%%%%%%%%%%%%%%%%%%%%%%%%%%%%%%%%%%%%%%%%%%%%%%%%%%%%%%%%%%%%%%%%%%%%%%%%%%%%%%%%%
%% APPENDIX
\newpage
\part{Causal Ground Truth}
\label{apx:causal_ground_truth}
\setcounter{section}{0}
In this appendix, we formalize a causal interpretation of the ground-truth graphs in \autoref{fig:ground_truths}. Equipped with this interpretation, which we give in \autoref{def:ground_truth}, the graphs describe constraints on the underlying causal system. We also provide a procedure to empirically validate these constraints using data collected from the chambers.

In \autoref{s:chambers}, we categorized the variables in each ground-truth graph $G$ as \emph{non-manipulable} variables, corresponding to sensor measurements, and \emph{manipulable} variables, that is, actuators and sensor parameters. Here, we further split the latter into manipulable variables with no incoming edges in $G$, which we call \emph{exogenous}, and manipulable variables with incoming edges, which we call \emph{endogenous} and are set by the control computer as a function of the other variables in the system. We refer to such functions as \emph{assignment functions}.

Let $X_V = (X_1,\ldots,X_k)$ denote a vector of chamber variables and let us assume that there is a true causal model describing the system; more formally, we denote by $P(X_j^t \mid \dop (X^{t_0}_{V \setminus \{j\}} = x))$ the distribution of variable $X_j$ at time $t$, after performing an intervention on all other chamber variables $X_{V\setminus \{j\}}$ at some earlier time $t_0 < t$. 
We can sample from this distribution by taking a measurement from the chamber, where the intervention corresponds to the following manipulations of its variables: if $X_i$ is a non-manipulable variable, \ie a sensor measurement, the intervention consists of overwriting its value with $x_i$. If $X_k \in X_{V\setminus \{j\}}$ is an exogenous manipulable variable,
then the intervention corresponds to setting it to $x_k$. If it is an endogenous manipulable variable, the intervention also sets its value to $x_k$, overriding the value set by its assignment function. 

\begin{definition}[Ground-truth causal model]
\label{def:ground_truth}
Let $G$ be a directed graph over $X_V$. 
We call $G$ a \emph{ground-truth graph} if
an edge $X_i \to X_j$ in $G$ implies that
\begin{align}
\label{eq:total_effect}
&\exists t, \wt > 0, x^A, x^B \text{ with } x^A_i \neq x^B_i \text{ and } x^A_k = x^B_k \text{ for all } k\neq i,\nonumber\\
&\text{such that } P(X_j^{t+\wt} \mid \dop (X_{V\setminus\{j\}}^t = x^A)) \neq P(X_j^{t+\wt} \mid \dop (X_{V\setminus\{j\}}^t = x^B)).
\end{align}
\end{definition}%

Although it is impossible to directly evaluate \eqref{eq:total_effect}---we would need to perform two different interventions simultaneously---\autoref{def:ground_truth} allows us, under additional assumptions, to make statements about 
what a ground-truth graph $G$ implies for the observations collected from the chambers. For example, if we assume that 
the distribution resulting from an intervention depends on the time lag (but not the absolute time $t$) at which the intervention is performed (the assumption is formalized in Proposition~\ref{prop:assumptions} below), then we can use two-sample tests for testing \eqref{eq:total_effect}. This assumption is suitable for the controlled environment of the light tunnel and for many of the effects in the wind tunnel. For the effects on the barometric measurements ($\tilde{P}_\text{up}, \tilde{P}_\text{dw}, \tilde{P}_\text{amb}, \tilde{P}_\text{int})$, whose distribution changes with time due to shifts in atmospheric pressure, the assumption is still reasonable for measurements taken closely in time.

\section{Empirical validation}
\label{ss:validation}

According to \autoref{def:ground_truth}, to validate an effect $X_i \to X_j$, we would need to reject the null hypothesis
\begin{align}
    \label{eq:null_theory}
    H_0 : &\forall t, \wt > 0, x^A, x^B \text{ with } x^A_i \neq x^B_i \text{ and } x^A_k = x^B_k \text{ for } k\neq i,\nonumber\\
&\text{such that } P(X_j^{t+\wt} \mid \dop (X_{V\setminus\{j\}}^t = x^A)) = P(X_j^{t+\wt} \mid \dop (X_{V\setminus\{j\}}^t = x^B)).
\end{align}
However, as discussed after \autoref{def:ground_truth}, additional assumptions are needed to test $H_0$, since we cannot perform the two different interventions $\dop (X_{V\setminus\{j\}}^t = x^A)$ and $\dop (X_{V\setminus\{j\}}^t = x^B)$ at the same time $t$. We first state our validation procedure and then provide sufficient conditions for it to have the appropriate level for testing $H_0$.

\paragraph{Validation procedure} 
Pick predefined $x_A,x_B,\wt,N,\alpha$.
\begin{enumerate}
    \item sample $U^1,\ldots,U^N$ i.i.d.\ from a Bernoulli distribution with $p=0.5$
    \item sample $\Delta t_1, \ldots, \Delta t_N \text{ i.i.d.} \sim \text{Unif}[10^{-3}, 1]$
    \item let $t_0 + \wt$ be the current time
    \item for $n=1,\ldots,N$
    \begin{enumerate}            
        \item wait $\Delta t_n$ seconds
        \item (at time $t_n := t_{n-1} + \wt + \Delta t_n$) set $X_{V \setminus \{j\}} \leftarrow x^A$ if $U^n=0$, or $X_{V \setminus \{j\}} \leftarrow x^B$ otherwise
        \item wait $\wt$ seconds
        \item (at time $t_n + \wt$) collect the measurement $x^{t_n + \wt}$ from the chamber
        \item add the measurement of $X_j$ to the sample $\bm{X}^A \leftarrow \bm{X}^A \cup \{x_j^{t_n + \wt}\}$ if $U^n=0$, and $\bm{X}^B \leftarrow \bm{X}^B \cup \{x_j^{t_n + \wt}\}$ otherwise
        \end{enumerate}
    \item perform a two-sample test $\psi$ that is level $\alpha$ (in practice, we use a  Kolmogorov-Smirnov \cite{smirnov1948table} test), comparing the samples $\bm{X}^A$ and $\bm{X}^B$, resulting in a p-value $p$
    \item reject $H_0$ if $p \leq \alpha$
    \end{enumerate}%
The above procedure amounts to a randomized controlled trial, where the treatment units are the $N$ measurements, and both the intervention allocations and intervention times are randomized. \autoref{prop:assumptions} provides sufficient conditions for the procedure to achieve the correct level in testing $H_0$ \eqref{eq:null_theory}.

\begin{proposition}
\label{prop:assumptions}
Let  $x_A,x_B,T,N,\alpha$ be the values used in the validation procedure and
assume that the following two assumptions hold.
\begin{enumerate}
\item[\emph{A1.}]
For all $t_1, \ldots, t_N$
such that $t_i + T < t_{i+1}$ for all $i \in \{1, \ldots, N-1\}$, and for all 
$x^1, \ldots, x^N \in \{x_A, x_B\}$,
it holds that for all $n \in \{1,\ldots,N\}$
$$
P(X_j^{t_n+\wt} \mid \dop (X_{V\setminus\{j\}}^{t_n} = x^n),
\ldots,
\dop (X_{V\setminus\{j\}}^{t_1} = x^{1}),
X_j^{t_{n-1}+\wt}, \ldots, X_j^{t_1+\wt})
= 
P(X_j^{t_n+\wt} \mid \dop (X_{V\setminus\{j\}}^{t_n} = x^n)).
$$

\item[\emph{A2.}] $\forall \tau, t$
$$P(X_j^{t+\wt} \mid \dop (X_{V\setminus\{j\}}^t = x_A)) = P(X_j^{t+\wt+\tau} \mid \dop (X_{V\setminus\{j\}}^{t+\tau} = x_A)),$$
and
$$P(X_j^{t+\wt} \mid \dop (X_{V\setminus\{j\}}^t = x_B)) = P(X_j^{t+\wt+\tau} \mid \dop (X_{V\setminus\{j\}}^{t+\tau} = x_B)).$$
\end{enumerate}
Then, the validation procedure is level $\alpha$ for testing $H_0$.
\end{proposition}
\begin{proof}
Assumptions A1 and A2 imply that there is $P^A$ and $P^B$ such that, conditioned on the realizations of $U_1, \ldots, U_N$,  
$\bm{X}^A \sim \text{i.i.d. } P^A$, and $\bm{X}^B \sim \text{i.i.d. } P^B$.
Additionally, under $H_0$, $P^A = P^B$. Thus,
$P(\psi(\bm{X}^A, \bm{X}^B)=1) \leq \alpha$.
\end{proof}

In essence, assumptions \emph{A1} and \emph{A2} ensure that the test in step 5 of the procedure receives two samples composed of independent and identically distributed observations. Alternative tests requiring weaker conditions may be possible. In any case, the assumptions are not unreasonable for the causal chambers. Both assumptions hold when $X_j$ is an endogenous manipulable variable, and they are reasonable for all sensor measurements in the light tunnel, and most in the wind tunnel. An exception comes from the barometric measurements $\tilde{P}_\text{dw}, \tilde{P}_\text{up},\tilde{P}_\text{amb}$ and $\tilde{P}_\text{int}$, which are affected by the atmospheric pressure outside the chamber. In this case, assumption \emph{A2} may not hold for large $\tau$. However, the issue is largely reduced for measurements taken very closely in time. Furthermore, the randomization of time points and intervention assignments provides some robustness against violations of A1 and A2.

The results of the procedure, in the form of the p-values computed in step 5 (using the Kolmogorov-Smirnov two-sample test), are shown for each edge and ground-truth graph in \autoref{tab:validation_lt_standard}-\autoref{tab:validation_wt_pressure_control}. The corresponding datasets are listed in the caption of each table, and the code to compute the p-values can be found in the Jupyter notebook \href{https://github.com/juangamella/causal-chamber-paper/blob/main/causal_validation.ipynb}{\nolinkurl{causal_validation.ipynb}} in the paper repository at \href{https://github.com/juangamella/causal-chamber-paper}{\nolinkurl{github.com/juangamella/causal-chamber-paper}}. The data from some experiments was used to compute the p-values for different edges and some experiments have been repeated a few times; in \autoref{tab:validation_lt_standard}-\autoref{tab:validation_wt_pressure_control} we show the raw p-values produced by the test without additional corrections for multiple testing.

\begin{table}[H]\centering
\ra{1.3}
\begin{tabular}{@{}rlrlrlrl@{}}\toprule
Edge & p-value & Edge & p-value & Edge & p-value & Edge & p-value\\
\midrule
$R\to\tilde{I}_1$ & $2.3 \times 10^{-14}\quad$& $G\to\tilde{I}_1$ & $1.6 \times 10^{-14}\quad$& $B\to\tilde{I}_1$ & $1.9 \times 10^{-14}\quad$& $R\to\tilde{I}_2$ & $2.3 \times 10^{-14}\quad$\\
$G\to\tilde{I}_2$ & $1.6 \times 10^{-14}\quad$& $B\to\tilde{I}_2$ & $1.9 \times 10^{-14}\quad$& $R\to\tilde{I}_3$ & $2.3 \times 10^{-14}\quad$& $G\to\tilde{I}_3$ & $1.6 \times 10^{-14}\quad$\\
$B\to\tilde{I}_3$ & $1.9 \times 10^{-14}\quad$& $R\to\tilde{V}_1$ & $2.3 \times 10^{-14}\quad$& $G\to\tilde{V}_1$ & $1.6 \times 10^{-14}\quad$& $B\to\tilde{V}_1$ & $1.9 \times 10^{-14}\quad$\\
$R\to\tilde{V}_2$ & $2.3 \times 10^{-14}\quad$& $G\to\tilde{V}_2$ & $1.6 \times 10^{-14}\quad$& $B\to\tilde{V}_2$ & $1.9 \times 10^{-14}\quad$& $R\to\tilde{V}_3$ & $2.3 \times 10^{-14}\quad$\\
$G\to\tilde{V}_3$ & $1.6 \times 10^{-14}\quad$& $B\to\tilde{V}_3$ & $1.9 \times 10^{-14}\quad$& $R\to\tilde{C}$ & $2.3 \times 10^{-14}\quad$& $G\to\tilde{C}$ & $1.6 \times 10^{-14}\quad$\\
$B\to\tilde{C}$ & $1.9 \times 10^{-14}\quad$& $\theta_1\to\tilde{I}_3$ & $3 \times 10^{-14}\quad$& $\theta_2\to\tilde{I}_3$ & $6.6 \times 10^{-14}\quad$& $\theta_1\to\tilde{V}_3$ & $3 \times 10^{-14}\quad$\\
$\theta_2\to\tilde{V}_3$ & $6.6 \times 10^{-14}\quad$& $\theta_1\to\tilde{\theta}_1$ & $3 \times 10^{-14}\quad$& $\theta_2\to\tilde{\theta}_2$ & $6.6 \times 10^{-14}\quad$& $R_1\to\tilde{\theta}_1$ & $1.6 \times 10^{-14}\quad$\\
$O_1\to\tilde{\theta}_1$ & $8.3 \times 10^{-101}\quad$& $R_2\to\tilde{\theta}_2$ & $2.3 \times 10^{-14}\quad$& $O_2\to\tilde{\theta}_2$ & $3.4 \times 10^{-67}\quad$& $R_C\to\tilde{C}$ & $1.9 \times 10^{-14}\quad$\\
$O_C\to\tilde{C}$ & $6 \times 10^{-24}\quad$& $L_{11}\to\tilde{I}_1$ & $1.6 \times 10^{-14}\quad$& $L_{12}\to\tilde{I}_1$ & $2.3 \times 10^{-14}\quad$& $L_{11}\to\tilde{V}_1$ & $1.6 \times 10^{-14}\quad$\\
$L_{12}\to\tilde{V}_1$ & $2.3 \times 10^{-14}\quad$& $T^I_1\to\tilde{I}_1$ & $1.9 \times 10^{-14}\quad$& $D^I_1\to\tilde{I}_1$ & $3 \times 10^{-14}\quad$& $T^V_1\to\tilde{V}_1$ & $1.9 \times 10^{-14}\quad$\\
$D^V_1\to\tilde{V}_1$ & $1.9 \times 10^{-14}\quad$& $L_{21}\to\tilde{I}_2$ & $4.2 \times 10^{-14}\quad$& $L_{22}\to\tilde{I}_2$ & $6.6 \times 10^{-14}\quad$& $L_{21}\to\tilde{V}_2$ & $4.2 \times 10^{-14}\quad$\\
$L_{22}\to\tilde{V}_2$ & $1.3 \times 10^{-12}\quad$& $T^I_2\to\tilde{I}_2$ & $1.9 \times 10^{-14}\quad$& $D^I_2\to\tilde{I}_2$ & $1.6 \times 10^{-14}\quad$& $T^V_2\to\tilde{V}_2$ & $1.6 \times 10^{-14}\quad$\\
$D^V_2\to\tilde{V}_2$ & $8.9 \times 10^{-13}\quad$& $L_{31}\to\tilde{I}_3$ & $2.3 \times 10^{-14}\quad$& $L_{32}\to\tilde{I}_3$ & $4.2 \times 10^{-14}\quad$& $L_{31}\to\tilde{V}_3$ & $6.8 \times 10^{-12}\quad$\\
$L_{32}\to\tilde{V}_3$ & $4.2 \times 10^{-14}\quad$& $T^I_3\to\tilde{I}_3$ & $2.3 \times 10^{-14}\quad$& $D^I_3\to\tilde{I}_3$ & $4.2 \times 10^{-14}\quad$& $T^V_3\to\tilde{V}_3$ & $1.9 \times 10^{-14}\quad$\\
$D^V_3\to\tilde{V}_3$ & $2.3 \times 10^{-14}\quad$&\\ 
\bottomrule
\end{tabular}
\caption{Results of the validation procedure for the edges of the standard-configuration graph of the light tunnel (\autoref{fig:ground_truths}a). For each edge, we show the p-value resulting from the Kolmogorov-Smirnov \parencite{smirnov1948table} test computed in step 5 of the procedure. All p-values are below $10^{-11}$. The experimental data, together with the interventions $x^A, x^B$, waiting time $T$ and sample size $N$ for each edge can be found in the \href{https://github.com/juangamella/causal-chamber/tree/main/datasets/lt_validate_v1}{\nolinkurl{lt_validate_v1}} dataset at \href{https://causalchamber.org}{\nolinkurl{causalchamber.org}}.}
\label{tab:validation_lt_standard}
\end{table}

\begin{table}[H]\centering
\ra{1.3}
\begin{tabular}{@{}rlrlrl@{}}\toprule
Edge & p-value & Edge & p-value & Edge & p-value\\
\midrule
$L_\text{in}\to\tilde{\omega}_\text{in}$ & $1.6 \times 10^{-14}\quad$& $T_\text{in}\to\tilde{\omega}_\text{in}$ & $1.4 \times 10^{-152}\quad$& $L_\text{in}\to\tilde{\omega}_\text{out}$ & $1.6 \times 10^{-14}\quad$\\
$L_\text{in}\to\tilde{C}_\text{in}$ & $1.6 \times 10^{-14}\quad$& $L_\text{in}\to\tilde{C}_\text{out}$ & $1.8 \times 10^{-22}\quad$& $L_\text{out}\to\tilde{\omega}_\text{in}$ & $3.4 \times 10^{-14}\quad$\\
$L_\text{out}\to\tilde{\omega}_\text{out}$ & $3.4 \times 10^{-14}\quad$& $T_\text{out}\to\tilde{\omega}_\text{out}$ & $8.3 \times 10^{-122}\quad$& $L_\text{out}\to\tilde{C}_\text{out}$ & $3.4 \times 10^{-14}\quad$\\
$L_\text{out}\to\tilde{C}_\text{in}$ & $1.1 \times 10^{-8}\quad$& $H\to\tilde{\omega}_\text{in}$ & $1.5 \times 10^{-17}\quad$& $H\to\tilde{\omega}_\text{out}$ & $4 \times 10^{-58}\quad$\\
$L_\text{in}\to\tilde{P}_\text{int}$ & $2.5 \times 10^{-7}\quad$& $H\to\tilde{P}_\text{int}$ & $8.8 \times 10^{-11}\quad$& $L_\text{out}\to\tilde{P}_\text{int}$ & $0.082\quad$\\
$O_\text{int}\to\tilde{P}_\text{int}$ & $1.1 \times 10^{-23}\quad$& $L_\text{in}\to\tilde{P}_\text{up}$ & $1.6 \times 10^{-14}\quad$& $H\to\tilde{P}_\text{up}$ & $3 \times 10^{-14}\quad$\\
$L_\text{out}\to\tilde{P}_\text{up}$ & $3.4 \times 10^{-14}\quad$& $O_\text{up}\to\tilde{P}_\text{up}$ & $8 \times 10^{-7}\quad$& $L_\text{in}\to\tilde{P}_\text{dw}$ & $1.6 \times 10^{-14}\quad$\\
$H\to\tilde{P}_\text{dw}$ & $3 \times 10^{-14}\quad$& $L_\text{out}\to\tilde{P}_\text{dw}$ & $3.4 \times 10^{-14}\quad$& $O_\text{dw}\to\tilde{P}_\text{dw}$ & $1.7 \times 10^{-5}\quad$\\
$O_\text{amb}\to\tilde{P}_\text{amb}$ & $1.5 \times 10^{-284}\quad$& $O_\text{in}\to\tilde{C}_\text{in}$ & $9.2 \times 10^{-22}\quad$& $R_\text{in}\to\tilde{C}_\text{in}$ & $2.8 \times 10^{-59}\quad$\\
$O_\text{out}\to\tilde{C}_\text{out}$ & $2.5 \times 10^{-22}\quad$& $R_\text{out}\to\tilde{C}_\text{out}$ & $5 \times 10^{-59}\quad$& $A_1\to\tilde{S}_1$ & $2.2 \times 10^{-59}\quad$\\
$O_1\to\tilde{S}_1$ & $6.9 \times 10^{-36}\quad$& $R_1\to\tilde{S}_1$ & $2.8 \times 10^{-59}\quad$& $A_1\to\tilde{S}_2$ & $1.1 \times 10^{-55}\quad$\\
$A_2\to\tilde{S}_2$ & $2.2 \times 10^{-59}\quad$& $O_2\to\tilde{S}_2$ & $9.5 \times 10^{-35}\quad$& $R_2\to\tilde{S}_2$ & $6 \times 10^{-59}\quad$\\
$A_1\to\tilde{M}$ & $4.5 \times 10^{-9}\quad$& $L_\text{in}\to\tilde{M}$ & $1.8 \times 10^{-9}\quad$& $L_\text{out}\to\tilde{M}$ & $2.7 \times 10^{-7}\quad$\\
$H\to\tilde{M}$ & $5.7 \times 10^{-16}\quad$& $O_M\to\tilde{M}$ & $1 \times 10^{-32}\quad$& $R_M\to\tilde{M}$ & $4.2 \times 10^{-59}\quad$\\
\bottomrule
\end{tabular}
\caption{
Results of the validation procedure for the edges of the standard-configuration graph of the wind tunnel (\autoref{fig:ground_truths}c). For each edge, we show the p-value resulting from the Kolmogorov-Smirnov \parencite{smirnov1948table} test computed in step 5 of the procedure. The computed p-values are all below $8 \times 10^{-7}$, except for the one corresponding to the edge $L_\text{out} \to \tilde{P}_\text{int}$ (0.082). The data to validate this edge was collected during a windy day (16.4.2024 in Zurich), and the effect on the intake barometer (a few pascals) was small compared to the large fluctuations (100 pascals) in the ambient atmospheric pressure. When using a more powerful rank-sum test \parencite{mann1947test}), we obtained a p-value of $0.029$. The experimental data, together with the interventions $x^A, x^B$, waiting time $T$ and sample size $N$ for each edge can be found in the \href{https://github.com/juangamella/causal-chamber/tree/main/datasets/wt_validate_v1}{\nolinkurl{wt_validate_v1}} dataset at \href{https://causalchamber.org}{\nolinkurl{causalchamber.org}}.}
\label{tab:validation_wt_standard}
\end{table}

\begin{table}[H]\centering
\ra{1.3}
\begin{tabular}{@{}rlrlrlrl@{}}\toprule
Edge & p-value & Edge & p-value & Edge & p-value & Edge & p-value\\
\midrule
$\theta_1\to\tilde{\text{I}}\text{m}$ & $2 \times 10^{-13}\quad$& $\theta_2\to\tilde{\text{I}}\text{m}$ & $1.6 \times 10^{-14}\quad$& $R\to\tilde{\text{I}}\text{m}$ & $2.3 \times 10^{-14}\quad$ & $G\to\tilde{\text{I}}\text{m}$ & $1.6 \times 10^{-14}\quad$\\
$B\to\tilde{\text{I}}\text{m}$ & $1.9 \times 10^{-14}\quad$& $T_\text{Im}\to\tilde{\text{I}}\text{m}$ & $3 \times 10^{-14}\quad$& $\text{Ap}\to\tilde{\text{I}}\text{m}$ & $2.3 \times 10^{-14}\quad$ &$\text{ISO}\to\tilde{\text{I}}\text{m}$ & $4.2 \times 10^{-14}\quad$\\
\bottomrule
\end{tabular}
\caption{Results of the validation procedure for the additional edges in the ``camera'' graph of the light tunnel (\autoref{fig:ground_truths}b). For each edge, we show the p-value resulting from the Kolmogorov-Smirnov \parencite{smirnov1948table} test computed in step 5 of the procedure. All p-values are below $10^{-12}$. For the image variable $\tilde{\text{I}}\text{m}$, the test is performed on the average of pixel values. The experimental data, together with the interventions $x^A, x^B$, waiting time $T$ and sample size $N$ for each edge can be found in the \href{https://github.com/juangamella/causal-chamber/tree/main/datasets/lt_camera_validate_v1}{\nolinkurl{lt_camera_validate_v1}} dataset at \href{https://causalchamber.org}{\nolinkurl{causalchamber.org}}.}
\label{tab:validation_lt_camera}
\end{table}

\begin{table}[H]\centering
\ra{1.3}
\begin{tabular}{@{}rl@{}}\toprule
Edge & p-value\\
\midrule
$\tilde{P}_\text{dw}\to L_\text{in}$ & $2.3 \times 10^{-14}\quad$\\
$\tilde{P}_\text{dw}\to L_\text{out}$ & $2.3 \times 10^{-14}\quad$\\
\bottomrule
\end{tabular}
\caption{Results of the validation procedure for the additional edges of the ``pressure-control'' graph of the wind tunnel (\autoref{fig:ground_truths}d). For each edge, we show the p-value resulting from the Kolmogorov-Smirnov \parencite{smirnov1948table} test computed in step 5 of the procedure. The experimental data, together with the interventions $x^A, x^B$, waiting time $T$ and sample size $N$ for each edge can be found in the \href{https://github.com/juangamella/causal-chamber/tree/main/datasets/wt_pc_validate_v1}{\nolinkurl{wt_pc_validate_v1}} dataset at \href{https://causalchamber.org}{\nolinkurl{causalchamber.org}}.}
\label{tab:validation_wt_pressure_control}
\end{table}
\vfill

%%%%%%%%%%%%%%%%%%%%%%%%%%%%%%%%%%%%%%%%%%%%%%%%%%%%%%%%%%%%%%%%%%%%%%%%%%%%%%%%%
%% APPENDIX
\part{Component Datasheets}
\label{apx:datasheets}

\setcounter{section}{0}

\begin{table}[H]\centering
\ra{1.3}
\begin{tabular}{@{}llp{14cm}@{}}\toprule
Component & Chamber & Description and datasheets\\
\midrule
%%%%
\texttt{light\_source} & \texttt{wt/lt} &
The light source of the light tunnel.
\begin{itemize}
    \item[-] LED array: \href{https://github.com/juangamella/causal-chamber/blob/main/hardware/datasheets/light_source.pdf}{\nolinkurl{datasheets/light_source.pdf}}
    \item[-] Individual LED (e.g. for wavelengths): \href{https://github.com/juangamella/causal-chamber/blob/main/hardware/datasheets/light_source_led.pdf}{\nolinkurl{datasheets/light_source_led.pdf}}
\end{itemize}
\\
%%%%
\texttt{current\_sensor} & \texttt{wt/lt} & Current sensor for the light source and fans. \href{https://github.com/juangamella/causal-chamber/blob/main/hardware/datasheets/current_sensor.pdf}{\nolinkurl{datasheets/current_sensor.pdf}}\\
%%%%
\texttt{motor} & \texttt{wt/lt} & The stepper motors to control the polarizers and the hatch.
\begin{itemize}
    \item [-] Motor: Model 17HS8401 in datasheet \href{https://github.com/juangamella/causal-chamber/blob/main/hardware/datasheets/motor.pdf}{\nolinkurl{datasheets/motor.pdf}}
    \item [-] A4988 motor driver: \href{https://github.com/juangamella/causal-chamber/blob/main/hardware/datasheets/motor_driver.pdf}{\nolinkurl{datasheets/motor_driver.pdf}}
\end{itemize}\\
%%%%
\texttt{angle\_sensor} & \texttt{lt} & The rotary potentiometer used to measure the polarizer angle. Model 3590S-6-502L in datasheet \href{https://github.com/juangamella/causal-chamber/blob/main/hardware/datasheets/angle_sensor.pdf}{\nolinkurl{datasheets/angle_sensor.pdf}}\\
%%%%
\texttt{light\_sensor} & \texttt{lt} & The light-intensity sensor of the light tunnel. Model Si1151-AB00-GMR in datasheet \href{https://github.com/juangamella/causal-chamber/blob/main/hardware/datasheets/light_sensor.pdf}{\nolinkurl{datasheets/light_sensor.pdf}}\\
%%%%
\texttt{led} & \texttt{lt} & The LEDs placed by each light sensor in the light tunnel.
\begin{itemize}
    \item [-] LEDs: \href{https://github.com/juangamella/causal-chamber/blob/main/hardware/datasheets/led.pdf}{\nolinkurl{datasheets/led.pdf}}
    \item [-] Digital rheostats: model MCP4151 in \href{https://github.com/juangamella/causal-chamber/blob/main/hardware/datasheets/potentiometer.pdf}{\nolinkurl{datasheets/potentiometer.pdf}}
\end{itemize}\\
%%%%
\texttt{camera} & \texttt{lt} & The light-tunnel camera, a Sony $\alpha$6100: \href{https://github.com/juangamella/causal-chamber/blob/main/hardware/datasheets/camera.pdf}{\nolinkurl{datasheets/camera.pdf}}\\
%%%%
\texttt{arduino} & \texttt{wt/lt} & The control computer used in the chambers, \ie an Arduino Mega Rev3: \href{https://github.com/juangamella/causal-chamber/blob/main/hardware/datasheets/arduino_mega.pdf}{\nolinkurl{datasheets/arduino_mega.pdf}}\\
%%%%
\texttt{fan} & \texttt{wt} & The wind tunnel fans: \begin{itemize}
    \item [-] Fan specifications: \href{https://github.com/juangamella/causal-chamber/blob/main/hardware/datasheets/fan.pdf}{\nolinkurl{datasheets/fan.pdf}}
    \item [-] Speed control white paper: \href{https://github.com/juangamella/causal-chamber/blob/main/hardware/datasheets/fan_pwm.pdf}{\nolinkurl{datasheets/fan_pwm.pdf}}
    \item [-] Specifications of a similar motor: \href{https://github.com/juangamella/causal-chamber/blob/main/hardware/datasheets/fan_motor.pdf}{\nolinkurl{datasheets/fan_motor.pdf}}
\end{itemize}\\
%%%%
\texttt{barometer} & \texttt{wt} & The high precision barometers used in the wind tunnel: \href{https://github.com/juangamella/causal-chamber/blob/main/hardware/datasheets/barometer.pdf}{\nolinkurl{datasheets/barometer.pdf}}\\
%%%%
\texttt{potentiometer} & \texttt{wt} & Digital potentiometer used to regulate amplitude in the amplification circuit of the wind tunnel. It is the same component as the rheostats of the light tunnel, \ie model MCP4151 in \href{https://github.com/juangamella/causal-chamber/blob/main/hardware/datasheets/potentiometer.pdf}{\nolinkurl{datasheets/potentiometer.pdf}}\\
%%%%
\texttt{microphone} & \texttt{wt} & The wind tunnel microphone consists of an electret microphone with the LM385-DIP8 amplifier in \href{https://github.com/juangamella/causal-chamber/blob/main/hardware/datasheets/microphone_amplifier.pdf}{\nolinkurl{datasheets/microphone_amplifier.pdf}}\\
%%%%
\texttt{speaker} & \texttt{wt} & The wind-tunnel speaker.
\begin{itemize}
    \item[-] Amplifier: model LM386-D08-T in \href{https://github.com/juangamella/causal-chamber/blob/main/hardware/datasheets/speaker_amplifier.pdf}{\nolinkurl{datasheets/speaker_amplifier.pdf}}
    \item[-] Speaker: \href{https://github.com/juangamella/causal-chamber/blob/main/hardware/datasheets/speaker.pdf}{\nolinkurl{datasheets/speaker.pdf}}
\end{itemize}\\
\bottomrule
\end{tabular}
\caption{Physical components of the chambers and their related datasheets. The datasheets can be found at \href{https://github.com/juangamella/causal-chamber/tree/main/hardware/datasheets}{\nolinkurl{github.com/juangamella/causal-chamber/}}.}
\label{tab:datasheets}
\end{table}

%%%%%%%%%%%%%%%%%%%%%%%%%%%%%%%%%%%%%%%%%%%%%%%%%%%%%%%%%%%%%%%%%%%%%%%%%%%%%%%%%
%% SECTION
\part{Additional Figures}
\label{apx:figures}
\setcounter{section}{0}

\begin{figure}[h]
\centerline{
\includegraphics[width=180mm]{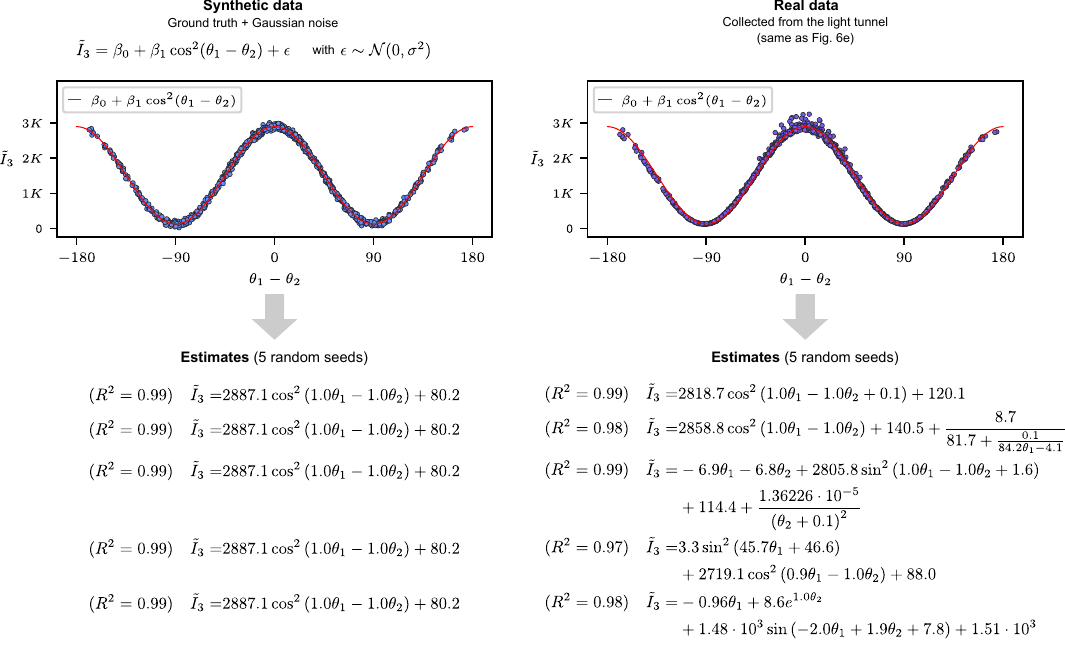}
}
\caption{Estimated expressions and their $R^2$ scores when we apply the symbolic regression method from \autoref{fig:benchmarks_2}e to the real data from the light tunnel (right), and synthetic data (left) following Malus' law (see \cref{ss:models_malus}). The synthetic data is produced by fitting the law to the data and adding Gaussian noise to simulate sensor noise. For the real data, the estimated expression varies with the random initialization of the method, whereas for the synthetic data, the method recovers the ground-truth expression in each run. This phenomenon does not carry over to the task of recovering Bernoulli's principle, where the method output is highly variable for both the synthetic and real data (see \autoref{fig:synthetic_symbolic_bern}).}
\label{fig:synthetic_symbolic_malus}
\end{figure}%

\begin{figure}[h]
\centerline{
\includegraphics[width=180mm]{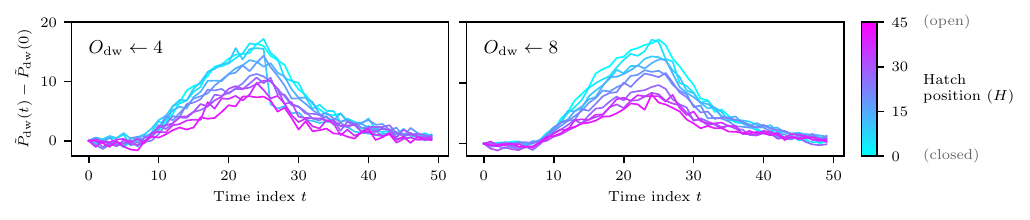}
}
\caption{Visualization of the pressure-curve dataset used in task b3 of \autoref{fig:benchmarks_1}. We show 10 curves picked at random for different positions of the hatch, under the oversampling rate seen during training ($O_\text{dw} \leftarrow 4$, left) and the one used as an out-of-distribution test ($O_\text{dw} \leftarrow 8$, right). The oversampling rate determines how many barometer readings are averaged to produce a single measurement (see \autoref{tab:wind_tunnel_vars}), increasing or decreasing its precision. While the effect on the signal is almost indistinguishable, the change is enough to cause the MLP in task b3 to fail in its out-of-distribution predictions.}
\label{fig:osr_comparison}
\end{figure}%
\vfill

\begin{figure}[h]
\centerline{
\includegraphics[width=180mm]{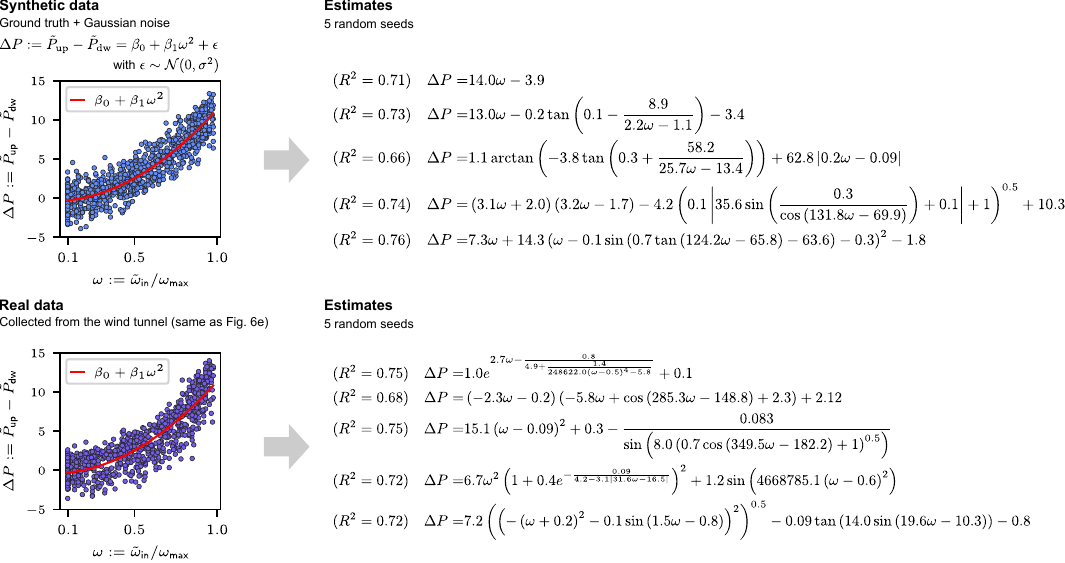}
}
\caption{Estimated expressions and their $R^2$ scores when we apply the symbolic regression method from \autoref{fig:benchmarks_2}e to recover Bernoulli's principle from synthetic data (top) and real data collected from the wind tunnel (bottom). The synthetic data is produced by fitting Bernoulli's principle to the data (see \cref{ss:models_bernoulli}) and adding Gaussian noise to simulate sensor noise. For both datasets, the output varies with the random initialization of the method.}
\label{fig:synthetic_symbolic_bern}
\end{figure}%
\vfill

\end{document}